\documentclass{article} % For LaTeX2e
\usepackage[accepted]{icml22/icml2022}

\usepackage{microtype}
\usepackage{graphicx}
\usepackage{booktabs} % for professional tables

\usepackage{hyperref}

\usepackage{amssymb, amsthm}
\usepackage{mathtools}
\usepackage{todonotes} % Optional math commands from https://github.com/goodfeli/dlbook_notation.
\usepackage{amsmath,amsfonts,bm}

\def\eqref#1{equation~\ref{#1}}
\def\1{\bm{1}}

\DeclareMathAlphabet{\mathsfit}{\encodingdefault}{\sfdefault}{m}{sl}
\SetMathAlphabet{\mathsfit}{bold}{\encodingdefault}{\sfdefault}{bx}{n}

\def\gL{{\mathcal{L}}}

\def\gN{{\mathcal{N}}}

\newcommand{\E}{\mathbb{E}}

\newcommand{\R}{\mathbb{R}}

\usepackage{thmtools}
\usepackage{thm-restate}
\usepackage{xcolor,colortbl}
\usepackage{url}
\usepackage{subfiles}
\usepackage{wrapfig}
\usepackage{multirow}
\usepackage{float}
\usepackage{enumitem}
\usepackage{subfig}
\usepackage{authblk}
\usepackage{xspace}
\usepackage[capitalise]{cleveref}
\usepackage{listings}
\allowdisplaybreaks[4]

\newcommand{\inlinefigbottomvspace}{{\vspace{-10pt}}} 
\newcommand{\inlinefigcaptionvspace}{{\vspace{-10pt}}}

\newcommand{\secvspace}{{\vspace{-5pt}}}
\newcommand{\subsecvspace}{{\vspace{-5pt}}}

\newtheorem{proposition}{Proposition}[section]

\newcommand{\ind}{\text{Index}} % {\mathbf{Index}}
\newcommand{\map}{m} % {\mathbf{m}}
\newcommand{\proj}{M} %{\mathbf{M}}
\newcommand{\lift}{M^{+}} %{\mathbf{M^{+}}}
\newcommand{\graph}{\mathcal{G}} 
\newcommand{\node}{\mathcal{V}}
\newcommand{\edge}{\mathcal{E}}

\newcommand{\encinvout}{H^{\phi} }
\newcommand{\priorinvout}{H^{\psi} }
\newcommand{\decequiout}{V^{\theta}}

\newcommand{\reqone}{\textbf{R1}\xspace}
\newcommand{\reqtwo}{\textbf{R2}\xspace}

\newcommand{\tx}{\widetilde{x}}
\newcommand{\dtx}{ {\Delta \widetilde{x}}}
\newcommand{\dtxi}{ {\Delta \widetilde{x}_i}}
\newcommand{\enc}{\mathbf{Enc}_\phi}
\newcommand{\dec}{\mathbf{Dec}_\theta}
\newcommand{\decoder}{\textnormal{Dec}}
\newcommand{\rbf}{\textnormal{RBF}}
\newcommand{\prior}{\mathbf{Prior}_\psi}

\newcommand{\dcut}{d_{\textnormal{cut}}}
\newcommand{\Dcut}{D_{\textnormal{cut}}}
\newcommand{\rmsdgen}{\textnormal{RMSD}_{\textnormal{gen}}}

\newcommand{\msg}{\textnormal{Msg}}

\newcommand{\update}{\textnormal{Update}}
\newcommand{\eupdate}{\sigma}

\usepackage{array, tabularx, caption, boldline}
\usepackage{graphicx}
\usepackage{cellspace}
\definecolor{codegreen}{rgb}{0,0.6,0}
\definecolor{codegray}{rgb}{0.5,0.5,0.5}
\definecolor{codepurple}{rgb}{0.58,0,0.82}
\definecolor{backcolour}{rgb}{0.95,0.95,0.92}

\lstdefinestyle{mystyle}{
    backgroundcolor=\color{backcolour},   
    commentstyle=\color{codegreen},
    keywordstyle=\color{magenta},
    numberstyle=\tiny\color{codegray},
    stringstyle=\color{codepurple},
    basicstyle=\ttfamily\footnotesize,
    breakatwhitespace=false,         
    breaklines=true,                 
    captionpos=b,                    
    keepspaces=true,                 
    numbers=left,                    
    numbersep=5pt,                  
    showspaces=false,                
    showstringspaces=false,
    showtabs=false,                  
    tabsize=2
}

\usepackage{amsmath}
\makeatletter
\let\save@mathaccent\mathaccent
\newcommand*\if@single[3]{%
  \setbox0\hbox{${\mathaccent"0362{#1}}^H$}%
  \setbox2\hbox{${\mathaccent"0362{\kern0pt#1}}^H$}%
  \ifdim\ht0=\ht2 #3\else #2\fi
  }
\newcommand*\rel@kern[1]{\kern#1\dimexpr\macc@kerna}
\newcommand*\widebar[1]{\@ifnextchar^{{\wide@bar{#1}{0}}}{\wide@bar{#1}{1}}}
\newcommand*\wide@bar[2]{\if@single{#1}{\wide@bar@{#1}{#2}{1}}{\wide@bar@{#1}{#2}{2}}}
\newcommand*\wide@bar@[3]{%
  \begingroup
  \def\mathaccent##1##2{%
    \let\mathaccent\save@mathaccent
    \if#32 \let\macc@nucleus\first@char \fi
    \setbox\z@\hbox{$\macc@style{\macc@nucleus}_{}$}%
    \setbox\tw@\hbox{$\macc@style{\macc@nucleus}{}_{}$}%
    \dimen@\wd\tw@
    \advance\dimen@-\wd\z@
    \divide\dimen@ 3
    \@tempdima\wd\tw@
    \advance\@tempdima-\scriptspace
    \divide\@tempdima 10
    \advance\dimen@-\@tempdima
    \ifdim\dimen@>\z@ \dimen@0pt\fi
    \rel@kern{0.6}\kern-\dimen@
    \if#31
      \overline{\rel@kern{-0.6}\kern\dimen@\macc@nucleus\rel@kern{0.4}\kern\dimen@}%
      \advance\dimen@0.4\dimexpr\macc@kerna
      \let\final@kern#2%
      \ifdim\dimen@<\z@ \let\final@kern1\fi
      \if\final@kern1 \kern-\dimen@\fi
    \else
      \overline{\rel@kern{-0.6}\kern\dimen@#1}%
    \fi
  }%
  \macc@depth\@ne
  \let\math@bgroup\@empty \let\math@egroup\macc@set@skewchar
  \mathsurround\z@ \frozen@everymath{\mathgroup\macc@group\relax}%
  \macc@set@skewchar\relax
  \let\mathaccentV\macc@nested@a
  \if#31
    \macc@nested@a\relax111{#1}%
  \else
    \def\gobble@till@marker##1\endmarker{}%
    \futurelet\first@char\gobble@till@marker#1\endmarker
    \ifcat\noexpand\first@char A\else
      \def\first@char{}%
    \fi
    \macc@nested@a\relax111{\first@char}%
  \fi
  \endgroup
}
\makeatother

\icmltitlerunning{Generative Coarse-Graining of Molecular Conformations}

\begin{document}

\twocolumn[
\icmltitle{Generative Coarse-Graining of Molecular Conformations}




\begin{icmlauthorlist}
\icmlauthor{Wujie Wang}{mit}
\icmlauthor{Minkai Xu}{mila,montreal}
\icmlauthor{Chen Cai}{ucsd}
\icmlauthor{Benjamin Kurt Miller}{uva}
\icmlauthor{Tess Smidt}{mit}
\icmlauthor{Yusu Wang}{ucsd}
\icmlauthor{Jian Tang}{mila,hec,cifar}
\icmlauthor{Rafael G\'omez-Bombarelli}{mit}
\end{icmlauthorlist}

\icmlaffiliation{mit}{Massachusetts Institute of Technology, USA}
\icmlaffiliation{mila}{Mila - Qu\'ebec AI Institute, Canada}
\icmlaffiliation{montreal}{Universite de Montr\'eal, Canada}
\icmlaffiliation{ucsd}{University of California San Diego, USA}
\icmlaffiliation{hec}{HEC Montr\'eal, Canada}
\icmlaffiliation{cifar}{CIFAR AI Chair, Canada}
\icmlaffiliation{uva}{University of Amsterdam, Netherlands}

\icmlcorrespondingauthor{Rafael G\'omez-Bombarelli}{rafagb.mit.edu}
% \icmlcorrespondingauthor{Eee Pppp}{ep@eden.co.uk}

\icmlkeywords{Molecular Dynamics, Coarse-Graining, Generative Models, Equivariance}

\vskip 0.3in
]



\printAffiliationsAndNotice{}  % leave blank if no need to mention equal contribution

%%%%%%%%% ICML template added by Chen above %%%%%%%%%%

% \title{Generative Coarse-Graining of Molecular Conformations}

% \begin{document}
% \maketitle

\begin{abstract}
% \Jian{A general suggestion: in the beginning of each session, I suggest to use the full name of CG and FG.}
Coarse-graining (CG) of molecular simulations simplifies the particle representation by grouping selected atoms into pseudo-beads and drastically accelerates simulation. However, such CG procedure induces information losses, which makes accurate backmapping, i.e., restoring fine-grained (FG) coordinates from CG coordinates, a long-standing challenge. Inspired by the recent progress in generative models and equivariant networks, we propose a novel model that rigorously embeds the vital probabilistic nature and geometric consistency requirements of the backmapping transformation. Our model encodes the FG uncertainties into an invariant latent space and decodes them back to FG geometries via equivariant convolutions. To standardize the evaluation of this domain, we provide three comprehensive benchmarks based on molecular dynamics trajectories. Experiments show that our approach always recovers more realistic structures and outperforms existing data-driven methods with a significant margin.
\end{abstract}

\section{Introduction}
\secvspace 
Defined as an operation to perform dimension reductions over continuous distributions, coarse-graining (CG) was first proposed by Paul Ehrenfest and Tatyana Afanasyeva in their study of molecular chaos~\citep{ehrenfest1959collected}. After nearly a century of theoretical developments, CG has become a powerful technique to simplify many complicated problems in physics. Useful applications of CG include the renormalization group techniques for critical phenomena~\citep{kadanoff1966scaling, wilson1971renormalization} and the usage of collective variables to study reaction dynamics~\citep{laio2002escaping}.  In molecular simulations, CG refers to the simplification of the particle representation by lumping groups of original atoms $x$ into individual beads $X$ with a rule-based CG mapping. CG molecular dynamics (MD) can significantly speed up computational discovery over chemical spaces 
% from CG building blocks which can be enumerated rapidly through 
with simpler combination rules~\citep{marrink2007martini, menichetti2019drug}; for simulations of large molecules with hundreds of thousands of atoms such as biological and artificial polymers, CG MD allows accessing the long time scales of phenomena like protein folding or polymer reptation~\citep{levitt1975computer, kremer1988crossover}. However, such acceleration comes at the cost of losing fine-grained (FG) atomic details, which are important for studying properties and interactions of atom-level structures or continuing simulations at FG scale. How to accurately recover the FG structure $x$ from CG coordinates $X$ remains a challenging problem. 

% Bio-molecular and materials systems investigated with computer simulations are expensive due to the large number of particles required for computation and can be accelerated with coarse-grained (CG) simulations by simulating the coarse-grained structures obtained by lumping atoms into a set of Coarse-Grained(CG) particles. For computational investigation, such representations can accelerate the simulation of collective dynamics such as protein folding or relaxation of polymer structures. 

% Although coarse-grained simulations provide computational speedup, atomistic resolutions are often required even if simulations are performed at coarser levels for the study of the atomic level interactions and observables or continuing simulations at a higher resolution. This requires inverting the CG  operation to infer the positions of all fine-grained atoms from the coarse-grained positions, which is similar to the super-resolution task in the context of images. More formally, the task here is to model the conditional distribution of full atom coordinates ($x \in \mathbb{R}^{n \times 3}$) given the coarse-grained coordinates ($X \in \mathbb{R}^{N \times 3}$, $N < n$) and the corresponding CG  mapping ($\proj$), \textnormal{i.e.}, modeling the reverse generative procedure $p(x|X)$. Past fine-grained reconstruction methods use a deterministic reconstruction map (a matrix) followed by energy minimization or molecular simulations to generate the atomistic ensemble. we argue that there are limitations caused by these coupled procedures. 

\begin{figure}%{0.5\textwidth}
   % \inlinefigtopvspace %\vspace{-30pt}
    \includegraphics[width=\linewidth]{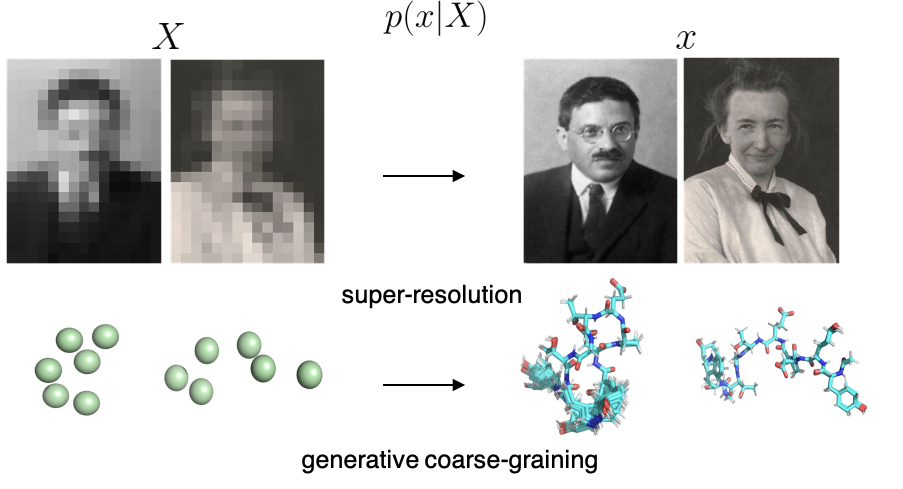}
    %\inlinefigbottomvspace %\vspace{-30pt}
    \caption[Caption for LOF]{An analogical illustration of the tackled generative coarse-graining problem. \textbf{Top}: Image super-resolution recovers higher resolution portraits\footnotemark. \textbf{Bottom}: Generative CG generates fine-grained molecular structures.}
    \label{fig:superres}
    \vspace{-10pt}
\end{figure}
\footnotetext{Portraits of P. Ehrenfest and T. Afanasyeva, taken from Dibner Library of the History of Science and Technology.}
% \end{figure}

So far, several methods have been proposed for the backmapping problem based on random projection~\citep{wassenaar2014going} or geometric rules~\citep{shih2007disassembly}. However, these methods usually can only produce poor initial geometries which require subsequent restraining molecular dynamics \cite{shih2007disassembly, brocos2012multiscale}, incurring considerable computational cost and large deviation from the original structures. These methods also require maintaining system-specific fragment libraries that are only applicable to a predefined mapping schema, preventing their usage from broader mapping choices and resolutions. Furthermore, 
none of these methods are data-driven, and geometries produced by these methods are biased toward predefined fragment geometries. 
To produce data-informed backampping solutions, several machine learning methods have recently been proposed \citep{wang2019coarse, an2020machine} to use parameterized functions to deterministically backmap. However, these methods also suffer from low-quality geometries and have not been tested on complex molecular structures. 
We argue that the problem is very challenging mainly because of the following:
% To build more powerful backmapping models with machine learning, we address a set of previously overlooked practical and theoretical considerations:
\textbf{(C1) Stochasticity of backmapping}: Due to surjective nature of CG projections, many different FG configurations can be mapped to the same CG conformation, hence the reverse generative map is one-to-many. However, this vital consideration on stochasticity has been overlooked in current deterministic methods~\citep{brocos2012multiscale, wassenaar2014going, wang2019coarse, an2020machine}, making it hard to generate diverse structures that capture the underlying FG distributions.
\textbf{(C2) Geometry consistency}: The backmapped FG coordinates should faithfully represent the underlying CG geometry, which requires that they can be reduced back to the original CG structures. Additionally, given the fact that CG transformation is $E(3)$ equivariant (see details in \cref{section:cg_def}), the backmapping function should also be $E(3)$ equivariant. A schematic illustration is given in \cref{fig:consistency}. However, these geometric consistency constraints have not been considered by existing methods. Failure to satisfy these requirements leads to unrealistic FG geometries that are not self-consistent. More discussions can be found in \cref{sec:consistency}.
\textbf{(C3) Generality w.r.t mapping protocols and resolutions:}
The choice of dimensionality reduction degree for FG to CG mappings varies wildly across CG modeling practices, depending on desired accuracy and simulation speed~\citep{marrink2007martini,  bereau2015automated, souza2021martini}. Therefore, it is desirable to have a general backmapping approach that is compatible with arbitrary CG mapping choices, for better usability under different CG simulation protocols.

\begin{figure}
\centering
\includegraphics[width=\linewidth]{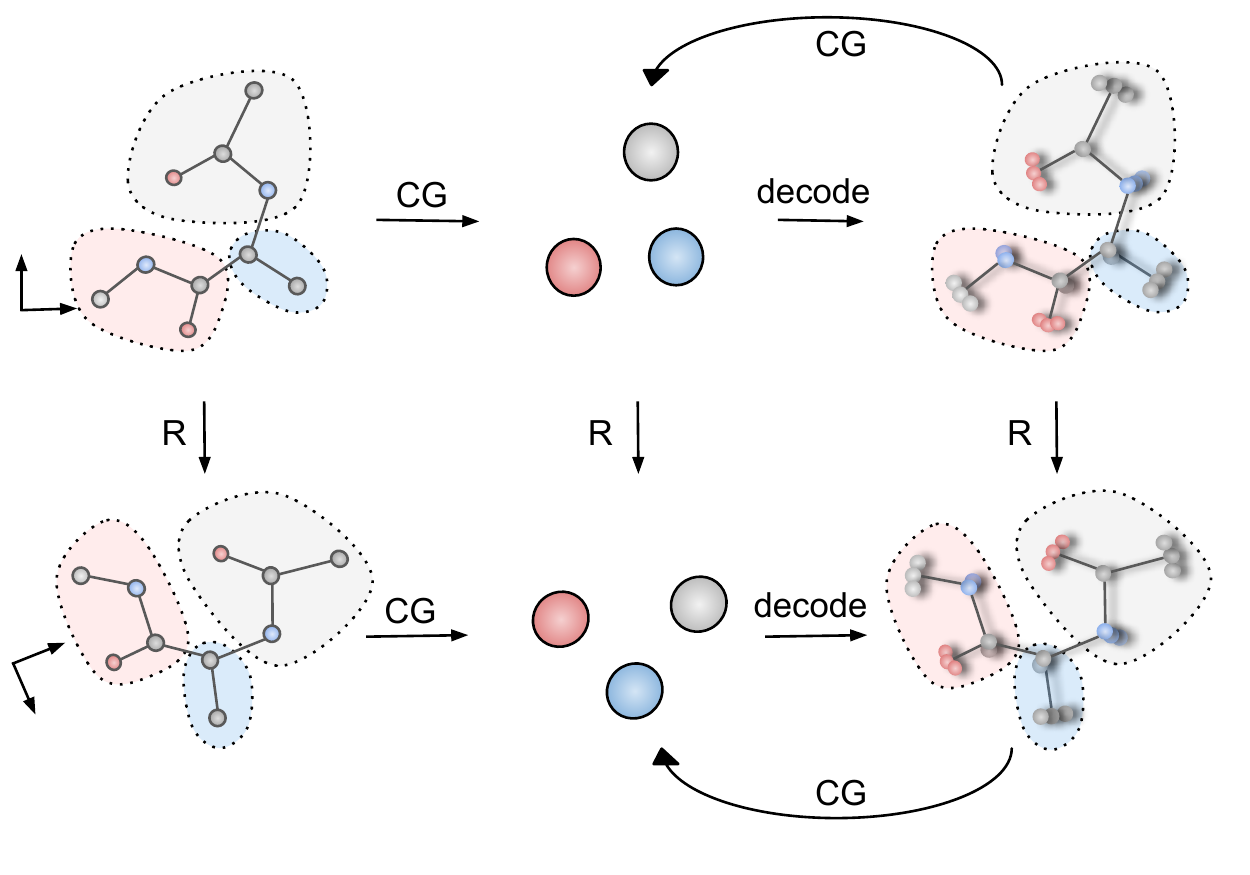}
\vspace{-20pt}
\caption{A diagram illustrating the geometric constraints for the backmapping function: 1) The backmapped coordinates should map back to the original CG coordinates; 2) because the CG transformation is equivariant, the backmapping transformation should also be equivariant.}
\inlinefigcaptionvspace
\label{fig:consistency}
\end{figure}

To address the requirements and challenges mentioned above, we propose a novel probabilistic model named Coarse-Graining Variational Auto-Encoder (CGVAE), a principled data-driven framework to infer atomistic coordinates from coarse-grained structures. We formulate the stochastic backmapping as a conditional generative task, i.e., modeling the conditional distribution $p_\theta(x|X)$ of FG molecular structures $x$ conditioned on CG structures $X$. To model the highly complex $p_\theta(x|X)$, we factorize the conditional distribution $p_\theta(x|X)$ as a latent variable model and parameterize a generative map for molecular geometries. This is achieved with $E(3)$ invariant encoding and equivariant decoding operations, using the message passing framework ~\citep{gilmer2017neural}. Our model generates coordinates in a one-shot fashion, requiring no predefined fragment geometry libraries to guide the backmapping.

Our contributions can be summarized as follows: 
\begin{itemize}[itemsep=0pt,topsep=0pt]
    \item We provide a principled probabilistic formulation for the backmapping problem of molecular conformations, and propose how to approximate this conditional distribution with a latent variable model in CGVAE (\textbf{C2}).
    \item We design a rigorous and expressive backmapping function that incorporates equivariant convolutions to satisfy the geometry consistency criteria (\textbf{C2}).
    \item CGVAE is CG mapping-agnostic (\textbf{C3}) and can be applied to diverse mapping protocols used in MD simulations in practice.
    \item To facilitate further developments in the field, we propose two benchmark datasets and formulate a suite of metrics to evaluate the quality of generated FG structures.
\end{itemize}
Experiments on proposed benchmarks show that our model can consistently achieve superior performance compared with previous data-driven methods.
% The results also demonstrate that our model can always generate both chemically accurate and structurally diverse structures even with extremely coarse regimes.

\section{Related Works}
\secvspace 
% \textbf{Backmapping of atomistic coordinates.} 
Reverse algorithms from CG to FG usually require two steps: backmapping projection followed by force field optimization ~\citep{guttenberg2013minimizing}. The projection typically relies on a deterministic linear backmapping operation \citep{wassenaar2014going, rzepiela2010reconstruction, brocos2012multiscale} or a parameterized deterministic function learned by supervised learning \citep{wang2019coarse, an2020machine}. However, these methods often produce invalid structures because they lack equivariance and chemical rule supervision, and thus rely on intensive force field equilibration of the backmapped systems \citep{rzepiela2010reconstruction, wassenaar2014going}. Recently, probabilistic backmapping (\textbf{C1}) has been incorporated to generate FG coordinates, but its use is limited to condensed phase systems with voxelized representation as model inputs \citep{stieffenhofer2020adversarial}.
To date, there have not been works that considered \textbf{C1-C3} all at once.

\section{Preliminaries}
\secvspace % \vspace{-5pt}
\subsection{CG Representations of Molecular Structures}
\label{section:cg_def}
\subsecvspace % \vspace{-5pt}
% \textbf{CG of Molecular Conformations.} 

We represent fine-grained (FG) systems as atomistic coordinates $x = \{x_i\}_{i=1}^n \in \R^{n \times 3}$. Similarly, the coarse-grained (CG) systems are represented by $X = \{X_i\}_{i=1}^N \in \R^{N \times 3}$ where $N < n$.
Let $ [n] $ and $[N]$ denote the set $\{1, 2, ..., n\}$ and $\{1, 2, ..., N\}$ respectively. CG operation of molecular simulations can be defined as an assignment $\map: [n] \rightarrow [N]$, which maps each FG atom $i$ in  $ [n] $ to CG atom $I \in [N]$, i.e., the CG node $I$ is composed from an ordered set of atoms $C_{I} = (k \in [n] \mid \map(k) = I )$. To ensure that particle mass and momentum are defined consistently with CG~\cite{noid2008multiscale}, each atom $i$ can only contribute to at most one $I$.  This CG projection operation can be represented as $X = \proj x$. $\proj \in \R^{N \times n}$ with $M_{I, i} = \frac{w_i}{\sum_{j \in C_I} w_j}$ if $i \in C_I$ and 0 otherwise. Here, $w_{i}$ is the projection weight of atom $i$. It can be either the atomic mass $i$ or simply $1$, which corresponds to placing $X_I$ at the center of mass or the center of geometry of $C_I$. The normalization term in $\proj$ ensures that coarse coordinates are the weighted averages of FG geometries. 

\subsection{Geometric Requirements of Backmapping}
\label{sec:consistency}
\subsecvspace % \vspace{-5pt}

% Let $H$ be a group acting on vector spaces $V$ and $W$, then a function $f:V\rightarrow \mathbb{R}$ is named $H$-invariant if $f(h\cdot x)=f(x)$ for all $x\in V,~ h\in H$. Similarly, a function $f:V\rightarrow W$ is named $H$-equivariant if the function commutes under the transformation, i.e., $f(h\cdot x)=h\cdot f(x)$ for all $x\in V$.
In this part, we discuss the symmetry properties of CG transformation to identify the geometric requirements to design the backmapping decoder function. First, we show that CG projection is $E(3)$ equivariant:

\begin{restatable}[]{property}{mequivariance}
\label{property:mequivariance}
    Let $f_\proj : \R^{n\times3} \rightarrow \R^{N\times3}$ be the linear transformation $f_\proj(x) := Mx$. $f_\proj$ is $E(3)$ equivariant, i.e., $f_\proj(Q x + g)=Q f_\proj(x) + g$, where $Q$ is a $3 \times 3$ orthogonal matrix, and $g$ is a translational vector.\footnote{Following \citet{satorras2021n}, here $Q x$  and $x + g$ are short-hands for $\{ Q x_1, ..., Q x_n \} $ and $\{ g + x_1, ..., g + x_n \} $ . The same convention also applies to $QX$ and $X+g$ at the coarse level. } 
\end{restatable}

The proof is provided in \cref{app:m-equivariance}. It states that the CG coordinates $X = \proj x$ will always translate, rotate, and reflect in the same way as $x$, obeying $E(3)$ equivariance. For a generic backmapping function $\decoder:\R^{N \times 3} \rightarrow \R^{n \times 3}$, it is natural to require that the backmapped coordinate $\tx$ can also be mapped back to $X$ for self-consistency, i.e. $\proj \decoder(X) = X$. As a corollary, such consistency requirements implies that if $f_M$ is equivariant, Dec should also be equivariant. (see \cref{app:consistency}). Formally, we consider the following requirements in our decoder design to ensure the geometric consistency discussed above:

\reqone{}. $\proj \tx = \proj \decoder(X) = X.$\\
\reqtwo{}. $\decoder(Q X + g)$ =  $Q \decoder(X) + g.$

To the best of our knowledge, these geometric requirements have not been considered in previous works. In \cref{section:decoder}, we introduce how our approach satisfies these requirements. 

\begin{figure*}[t]
\centering
    % \figtopvspace % \vspace{-10pt}
    \includegraphics[width=\textwidth]{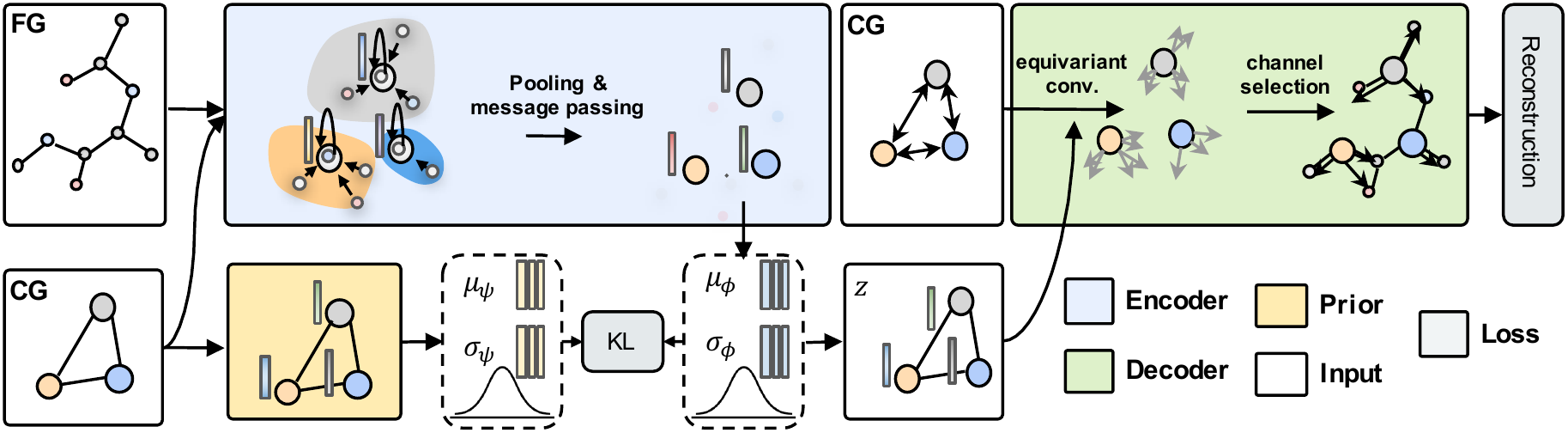}
    % \figbottomvspace 
    \vspace{-10pt}
    \caption{The overall framework of generative CG modeling. The encoder $q_\phi(z| x, X)$ takes both FG and CG structures as inputs, and outputs invariant latent distributions for each CG node via message passing and pooling. Given sampled latent variables and the CG structure, the decoder $p_\theta(x |X,z)$ learns to recover the FG structure through equivariant convolutions. 
    The whole model can be learned end-to-end by optimizing the KL divergence of latent distributions and reconstruction error of generated FG structures.
    }
    \label{fig:schematic}
    \vspace{-5pt}
\end{figure*}
% \vspace{-25pt}

\section{Method}
\secvspace
\label{section:method}
In this section, we introduce our probabilistic framework of generating fine-grained (FG) coordinates from coarse-grained (CG) coordinates. We first present an overview of the conditional generation formalism in \cref{section:probalistic_formulation} and describe the detailed model architecture in \cref{section:encoder,section:decoder}. Then we show the training and sampling protocol in \cref{section:training-sampling}. A notation table is available in \cref{table:symbol_notation} in the appendix.

\subsection{Generative Coarse-Graining Framework}
\label{section:probalistic_formulation}
\subsecvspace

The CG  procedure defines a surjective mapping from the FG coordinates $x$ to the CG coordinates $X$, and we study the inverse problem of recovering $x$ from $X$. This one-to-many problem is a generation task since the same CG system can backmap to different FG conformations. We propose to learn a parameterized conditional generative model $p_\theta(x|X)$ to approximate the recovering function. We further incorporate the uncertainty of FG coordinates into a continuous latent space $z$ and factorize the conditional distribution as an integral over $z$, i.e., $p(x | X) = \int p_\theta(x |X,z) p_\psi(z | X) dz$. This integral is known to be intractable~\citep{kingma2013auto}, and we instead maximize its variational lower bound with an approximate amortized posterior distribution $q_\phi(z|X,x)$~\citep{sohn2015learning}:
\begin{equation}
\begin{aligned}
    \log p(x|X) \geq & \underbrace{\E_{q_\phi(z | x, X)} \log p_\theta(x | X, z)}_{\gL_{\text{recon.}}}\\
    & + \underbrace{\E_{q_\phi(z | x, X)} \log \frac{ p_\psi(z|X) }{q_\phi(z| x, X)}}_{\gL_{\text{reg.}}} %:= \gL_{\text{ELBO}},
\label{eq:elbo}
\end{aligned}
\end{equation}
where $p_\theta(x|X,z)$ is the decoder model to recover $x$ from $X$ and $z$, and $q_\phi(z|x,X)$ and $p_\psi(z|X)$ are the encoder and prior model respectively to model the uncertainties from CG reduction. The first term $\gL_{\text{recon.}}$ is the expected reconstruction error of generated FG structures, and the second term $\gL_{\text{reg.}}$ can be viewed as a regularization over the latent space. A schematic of the whole framework is provided in \cref{fig:schematic}. In the following sections, we elaborate on how we estimate these distributions with graph neural networks. 
% \chen{i think we should emphasize somewhere (either here or later when introduce each components of CGVAE) that besides those standard formulation, we also need to satisfy some invariance/equivairnace constraints, which is new and requires some work.}

\subsection{CG Encoding and Prior}
\subsecvspace
\label{section:encoder}

\textbf{Overview of the data representation.} The encoding process extracts information from both the FG structure $\graph_{\textnormal{FG}}$ and CG structure $\graph_{\textnormal{CG}}$. $\graph_{\textnormal{FG}}$ represents radius graphs with a distance cutoff $\dcut$. Radius graph is a widely adopted choice that can differentiate conformations with the same molecular graph~\citep{axelrod2020molecular}.$\graph_{\textnormal{CG}}$ are induced graphs based on assignment map $\map$ as introduced in \cref{section:cg_def}. See \cref{app:graph_struct} for detailed descriptions of these graphs in our model. In short, the nodes are labeled with atomic types and the edges are labeled with distances. 

\textbf{Estimating the posterior distribution: $(X, x) \rightarrow z$.} The encoder model extracts CG-level invariant embeddings by three operations at each convolution step: 1) message passing at FG level, 2) pooling operation that maps FG space to CG space, and 3) message passing at CG level. After convolution updates, the obtained CG embeddings are then used to parameterize the latent distribution with CG latent variable $z \in \mathbb{R}^{N \times F}$ where $F$ is the dimension of features. For each operation, we use a SchNet-like architecture \citep{schutt2017schnet}. The first convolution process at FG level can be described using message passing neural networks (MPNN) \citep{gilmer2017neural}: 
\begin{equation}
\small
    h_i^{t+1} = \update\Big(\sum_{j \in N(i)} \msg(h_i^{t}, h_j^{t}, \rbf(d_{ij}) ), \; h_i^{t}\Big),
    \label{eq:fg2fg}
\end{equation}
where $N(i)$ denotes the neighbors of $i$ in $\graph_{\textnormal{FG}}$ and  $d_{ij}$ denotes the distance between atoms $i$ and $j$. $\rbf(\cdot)$ is the radial basis transformation introduced in \citet{klicpera2020directional, schutt2021equivariant}, which transforms the distances to high dimensional features. The initial node embeddings $h_i^0$ are parameterized from atomic types. Since both atomic types and edge distances are invariant features, the produced FG-wise features will also be invariant. The FG embeddings are then pooled to CG level based on the CG map $\map$: 
\begin{equation}
\small
    \widetilde{H}_{I}^t = \update\Big(\sum_{i \in C_{I}}\msg \left(h_i^{t+1}, H_I^{t}, \rbf(d_{iI}) \right) , \; H_I^{t} \Big),
    \label{eq:fg2cg}
\end{equation}
where $\widetilde{H}_{I}^t \in \R^F$ is the pooled CG embeddings from the FG graph and $d_{iI}$ is the distance between $x_i$ and $X_{I}$, i.e., $d_{iI} = \|x_i - X_{I}\|_2$; $H_I^0$ is initialized as $0_{F}$. This step pools messages from FG nodes to their assigned CG nodes, followed by an updating procedure. The CG-level node features are further parameterized with message passing at the CG level:
\begin{equation}
\small
   H_I^{t+1} = \update\Big(\sum_{J \in N(I)} \msg(\widetilde{H}_{I}^t, \widetilde{H}_{J}^t, \rbf(d_{IJ})), \;  H_{I}^t \Big),
   \label{eq:cg2cg}
\end{equation}
where $N(I)$ denotes the neighbors of $I$ in $\graph_{\textnormal{CG}}$, and $d_{IJ}$ denotes the distance between CG beads $I$ and $J$. 

The aforementioned three steps constitute the convolutional module for extracting both CG and FG level information, with their detailed design described in \cref{appendix:encoder_prior}.  After several steps of convolutional updates in our encoder, we can obtain the final embeddings $\enc(x,X) = \encinvout_I \in \R^F$ for each CG atom $I$. Then two feedforward networks $\mu_\phi$ and $\sigma_\phi$ are applied to parameterize the posterior distributions as diagonal multivariate Gaussians $\mathcal{N}(z_I |\mu_\phi(\encinvout_I), \sigma_\phi(\encinvout_I) )$. Note that, $\encinvout$ are designed as invariant features and thus the resulting likelihood is also invariant. 

\textbf{Estimating the prior distribution: $X \rightarrow z$.} The prior model $\prior(X)$ only takes the CG structure $\graph_{\textnormal{CG}}$ as inputs to model the latent space. The initial CG node features are set as the pooled sum of one-hot FG fingerprints $H^0_I = \sum_{i \in C_{I}} h_i^0$. Then we employ the same networks as defined in \cref{eq:cg2cg} to conduct graph convolutions over $\graph_{CG}$ and update the node embeddings.  After obtaining the final CG embeddings $\prior(X) = \priorinvout$, we also model $p_\psi(z| X)$ as the products of CG-wise normal distributions $\mathcal{N}( z_I | \mu_\psi(\priorinvout_I), \sigma_\psi(\priorinvout_I)$ with means and diagonal covariances parameterized from $\priorinvout_I$ using two separate neural networks $\mu_\psi$ and $\sigma_\psi$.
 
\subsection{Multi-channel Equivariant Decoding}
\subsecvspace
\label{section:decoder}

Our decoder design is inspired by recent advances in vector-based graph neural nets ~\citep{schutt2021equivariant, satorras2021n, jing2020learning, batzner2021se}. On the high level, the decoder $p_\theta(\tx |X, z)$ maps the coarse geometry $X \in \R^{N\times3}$ and invariant latent variables $z \in \R^{N \times F} $ to produce $\tx \in \R^{n \times 3 }$ as predictions of FG geometry $x$. Our proposed decoder consists of three steps 1) Given $X$ and invariant feature $z$ from the encoder, we will generate $F$ equivariant vector channel for each bead, resulting in $V^{\theta}\in \R^{N \times F \times 3}$. 2) As each bead corresponds to different number of atoms, not all channels will be used. From vector channels $V^{\theta}$, we select a proper number of channels for each bead to form $\dtx \in \R^{n \times 3}$ which is the prediction for the relative position $\Delta x$. 3) In the last step, we go from relative position $\dtx$ to absolute $\tx$. This step has to satisfy \textbf{R1}. We elaborate on these three steps next.

\textbf{Equivariant convolution: $(X, z) \rightarrow V^{\theta}$.} 
We use equivariant convolutions based on inter-node vectors to predict $\dtx$.  Apart from invariant node and edge features, our decoder takes a set of edge-wise equivariant vector features $ \{ \widehat{E}_{IJ}= \frac{X_J - X_I}{d_{IJ}} \mid (I, J) \in \edge_{\textnormal{CG}} \}$ as inputs, which form a basis set to construct $\Delta \tx$.  Additionally, our model also include pseudoscalars/pseudovectors which contain additional information of chirality. Unlike scalars/vectors, they undergo a sign flip under reflection transformation (see \cref{app:pseudo}). The message passing operations in the decoder update node-wise scalar features $H \in \R^F$, pseudoscalar features $\widebar{H} \in \R^F$, vector features $V \in \R^{F \times 3} $ , and pseudovector features $\widebar{V} \in \R^{F \times 3}$ separately. The convolutional updates in our decoder are performed as follows: 
\begin{equation}
\small
\begin{aligned}
\label{eq:dec}
    \Delta H_I^t = \sum_{J \in N(i)}& W_1 \circ H^t_J \\
    \Delta \widebar{H}_I^t = \sum_{J \in N(i)}& V^t_I \cdot \widebar{V}^t_J \\
        \Delta V_I^t = \sum_{J \in N(i)}& \Big( W_2 \circ (V^t_I \times \widebar{V}^t_J) + W_3 \circ \widebar{H}^t_I \circ \widebar{V}^t_J  \\
    &+ W_4 \widehat{E}_{IJ} + W_5 \circ V^t_J \Big)\\
        \Delta \widebar{V}_I^t = \sum_{J \in N(i)} & \Big( W_6 \circ (V^t_I \times V^t_J) + W_7 \circ (\widebar{V}^t_I \times \widebar{V}^t_J) \\ 
    & + W_8 \circ \widebar{V}^t_J + W_9 \circ  \widebar{H}^t \circ V^t_J \Big)
\end{aligned}
\end{equation}
% \begin{equation}
% \small
% \begin{aligned}
% \nonumber
%     \Delta \widebar{V}_I^t = \sum_{J \in N(i)} & \Big( W_6 \circ (V^t_I \times V^t_J) + W_7 \circ (\widebar{V}^7_I \times \widebar{V}^t_J) \\ 
%     & + W_8 \circ \widebar{V}^t_J + W_9 \circ  \widebar{H}^t \circ V^t_J \Big)
% \end{aligned}
% \end{equation}
where $\times$, $\circ$ and $\cdot$ are cross product, element-wise product and dot product operations respectively. $W_1$-$W_9$ are parameterized invariant edge-wise filters. Let $L_1: \R^F \rightarrow \R^F $ and $ L_2 : \R^K \rightarrow \R^F $ be neural networks, then each filter is implemented as as $W  \in \R^{F} = L_1(\rbf(d_{IJ})) \circ L_2(H_J)$. The message passing operations are followed by an equivariant update block $\eupdate(\cdot, \cdot)$ as proposed in PaiNN\citep{schutt2021equivariant}. It is used as an update block that linearly mixes vector channels and introduces non-linear coupling between $H$ and $V$. For pseudoscalar and pseudovector $\widebar{H}$ and $\widebar{V}$, we simply apply a residual update to ensure that they flip sign under reflection:
\begin{equation}
\small
\begin{aligned}
H_I^{t+1}, V_I^{t+1} &= \sigma(H_I^t + \Delta H_I^t, V_I^t + \Delta V_I^t) \\
\widebar{H}_I^{t+1}, \widebar{V}_I^{t+1} &= \widebar{H}_I^t + \Delta \widebar{H}_I^t, \widebar{V}_I^t + \Delta \widebar{V}_I^t
\end{aligned}
\label{eq:update}
\end{equation}
Pseudoscalars and pseudovectors in our decoder come from cross product updates. This introduces a richer basis set for the coordinate construction especially for low $N$ cases. When $N=3$, the span of the vector basis will be constrained in a plane, and cross product can overcome this limitation by introducing a vector basis in the orthogonal directions, increasing the expressiveness of the model. 

The initial $H_I^0$ is obtained from $z \sim q_{\phi}(z|X,x)$ for training, or $p_{\psi}(z|X)$ for sampling. We initialize other feature set with zero vectors, i.e., $\widebar{H}_I^0=0_{F}$, $V_I^0 = 0_{F \times 3}$, and $\widebar{V}_I^0 = 0_{F \times 3}$. For 3-bead CG geometries which are not chiral, we initialize $\widebar{H}_I^0=1_{F}$ to break the symmetry (see \cref{app:pseudo_init} for more discussions). In this way, information in the pseudovector channel can be propagated to vector channels to construct a basis set that spans 3D. 

After several layers of convolution updates, the decoder produces the final vector output $\decequiout \in \mathbb{R}^{N \times F \times 3}$. It is easy to show that $\decequiout$ transforms in an $E(3)$ equivariant way because it is updated strictly with vector operations, which satisfies \reqtwo in \cref{sec:consistency}. More discussions on the equivariance properties of the decoder are available at \cref{appendix:decoder} and \citet{schutt2021equivariant}.

\textbf{Channel selection: $V^{\theta} \rightarrow \dtx$.} In this part we introduce how to obtain the relative position predictions $\Delta \tx \in \R^{n \times 3 }$ from the equivariant vector outputs $\decequiout \in \R^{N \times F \times 3}$. 
This requires a channel selection procedure illustrated in \cref{fig:channel_selection}A.  As each atom $i$ is assigned to $I=\map(i)$, the corresponding prediction of position $\dtx_i$ is chosen to be the vector channel with the index of $i$ in $C_I$. This requires an index retrieval function $\ind(i, C_I)$ given that $C_I$ is ordered.\footnote{$\ind(i, C_I)$ retrieves the index of $i \in C_I$. For example, $\ind(1, (1,2,4)) = 0, \ind(4, (1,2,4)) = 2 $} The prediction of $\dtx$ does not utilize all the $F > |C_I|$ vector channels at the final vector output $\decequiout$, while all the channels participate in the convolutions before the read-out. We select a proper number of vector channels from $V^{\theta}$ and concatenate them to form the FG level prediction $\dtx$. More precisely,  $\dtx = \bigoplus_i \decequiout_{\map(i), \ind(i, C_I)}$ where $m(i)$ selects CG beads and $\textnormal{Index}(i, C_I)$ selects vector channel of bead $I$. This operation can be done conveniently in Pytorch \citep{pytorch}. See implementation details in  \cref{appendix:channel_select}.

\begin{figure}
\centering
\includegraphics[width=\linewidth]{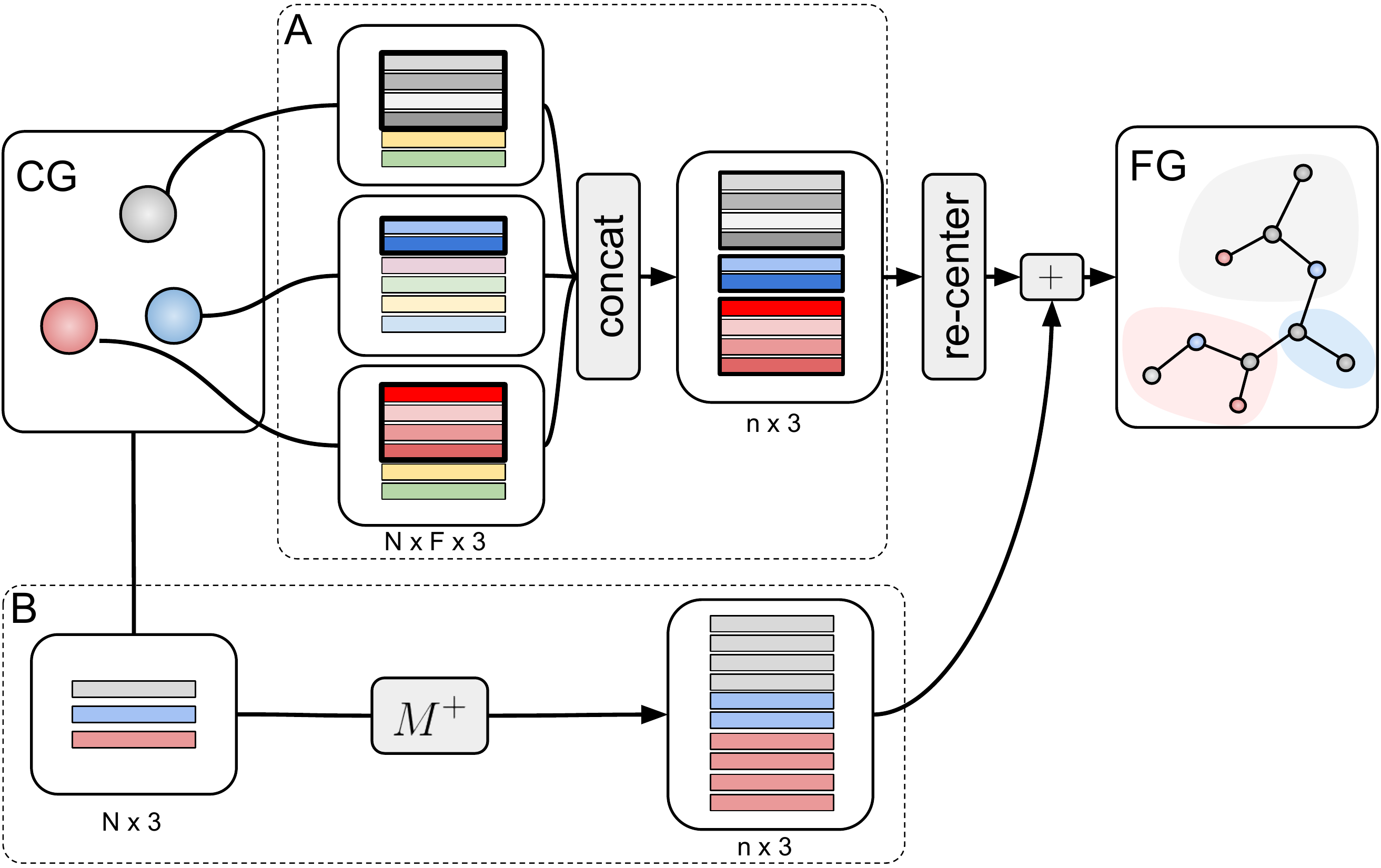}
\inlinefigbottomvspace % \vspace{-20pt}
\caption{A schematic diagram to describe how coordinates are reconstructed from vector channels. \textbf{A.} The prediction of $\dtx_i$ are compiled from selected vector channels of CG node $\map(i)$ based on their index in set $C_I$; this is followed by the concatenation of selected vectors to compile $\dtx = \bigoplus_i \decequiout_{\map(i), \ind(i, C_I)}$. The obtained coordinates are further re-centered by subtracting $\proj^+\proj \dtx$ so that $\reqtwo$ is satisfied. \textbf{B.} The CG coordinates are ``lifted" to the FG space via the lifting operator $\proj^+$ to get $\proj^+ X \in \R^{n \times 3}$ which are used as the geometric reference for constructions.}
% \inlinefigcaptionvspacer
\label{fig:channel_selection}
% \end{wrapfigure}
\end{figure}

% 
% \chen{ As $\dtx$ is the relative position, we need one more step to go from $\dtx$ to final FG coordinates $\tx$. 
% This is done by setting $\tx:=\lift X + \dtx - \lift \proj \dtx$. $\lift \in \R^{n \times N}$ is a lift operator that maps CG geometry $X$ back to $x$. $\lift_{i, I} = 1$ if $I = m(i)$ and 0 otherwise. It is the pseudoinverse of the projection operator $M$ and used in graph coarsening literature \citep{loukas2019graph, cai2020graph}. See \cref{app:lift} for more details. The term $-\lift \proj \dtx$ is added to ensure \reqone is satisfied, i.e., $M\dtx=x$. See \cref{appendix:decoder_compat} for a simple proof.}
% \chen{I will leave it as your choice.}

\textbf{Compile predictions for FG coordinates: $\dtx \rightarrow \tx$.}   As $\dtx$ is the relative position, we need one more step to go from $\dtx$ to final FG coordinates $\tx$. This is done by setting $\tx:=\lift X + \dtx - \lift \proj \dtx$.  $\lift \in \R^{n \times N}$ is the lift operator with $\proj^+_{i,I} = 1$ if $I = m(i)$ and 0 otherwise (\cref{app:lift}). Similar constructions can also be found in  graph-coarsening methods \citep{loukas2019graph, cai2020graph}. As described by \cref{fig:channel_selection}B, $\proj^+$ maps the point sets space from $\R^{N \times 3}$ to $\R^{n \times 3}$ by assigning or ``lifting'' CG coordinates back to their contributing FG atoms in $C_I$. The term $-\lift \proj \dtx$ is added to re-center $\dtx$ so that we can get the original $X$ back after CG projection in order to satisfy \reqone. See  \cref{appendix:decoder_compat} a formal statement with proof.

\subsection{Model Training and Sampling }
\subsecvspace
\label{section:training-sampling}

\textbf{Reconstruction loss calculation.} We calculate the reconstruction loss introduced in \cref{eq:elbo} as two invariant components: $\gL_{\textnormal{MSD}}$ and $\gL_{\textnormal{graph}}$. The mean square distance loss $\gL_{\textnormal{MSD}} = \frac{1}{n}\sum_{i=1}^n \|\widetilde{x}_i - x_i\|_2^2$ is a standard measurement to quantify differences between the generated reference coordinates~\citep{wang2019coarse,xu2021end}.  To supervise the model to produce valid distance geometries, we introduce another loss term $\gL_{\textnormal{graph}}$. Let $\edge$ denote the set of chemical bonds in the generated molecular graph. Following previous works~\citep{xu2021learning}, we also expand the set with some auxiliary multi-hop edges to fix rotatable bonds.
We also include all 2-hop neighbor pairs in $\edge$. Then the objective can be calculated as $\gL_{\textnormal{graph}} = \frac{1}{|\edge|} \sum_{(i, j) \in \edge} (d (\tx_i,\tx_j) - d(x_i, x_j))^2$, which can be viewed as a supervision over the bond lengths. The total reconstruction loss is computed as
$\gL_{\textnormal{recon.}} =  \gL_{\textnormal{MSD}} + \gamma \gL_{\textnormal{graph}}$, where $\gamma$ is a weighting hyperparameter to balance the two loss contributions.

\textbf{Training and sampling.} We adopt the training objective $\gL = \gL_{\text{recon.}} + \beta \gL_{\text{reg.}}$, where $\beta$ is a hyperparameter to control the regularization strength~\citep{higgins2016beta}. During training, the CG latent variable $z$ is sampled from $q_{\phi}(z| x, X)$ by $z = \mu_\phi + \sigma_\phi \circ \epsilon$, where $\epsilon \sim \gN (0,I)$. Then both $z$ and $X$ are taken as inputs of the decoding model to generate $\widetilde{x}$. For sampling, given the coarse structure $X$, we first sample the invariant latent variable from the prior model by $z \sim p_\psi(z|X)$, and then pass the sampled latent representation $z$ and the coarse coordinates $X$ into the decoder to generate FG coordinates $\widetilde{x} = \dec(X,z)$. 

\section{Experiments} 
\secvspace

\begin{figure*}[!tbp]
  \centering
  \vspace{-10pt}
  \subfloat{\includegraphics[width=0.45\textwidth]{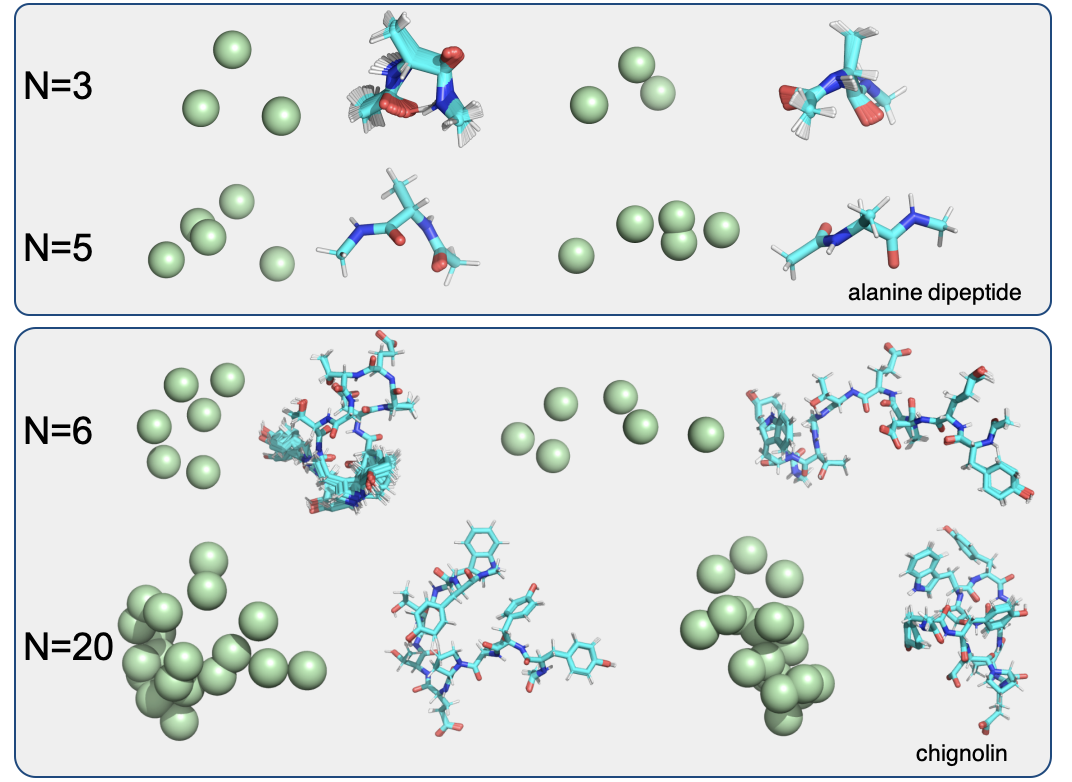}\label{fig:quant_results_f1}}
  \hfill
  \subfloat{\includegraphics[width=0.55\textwidth]{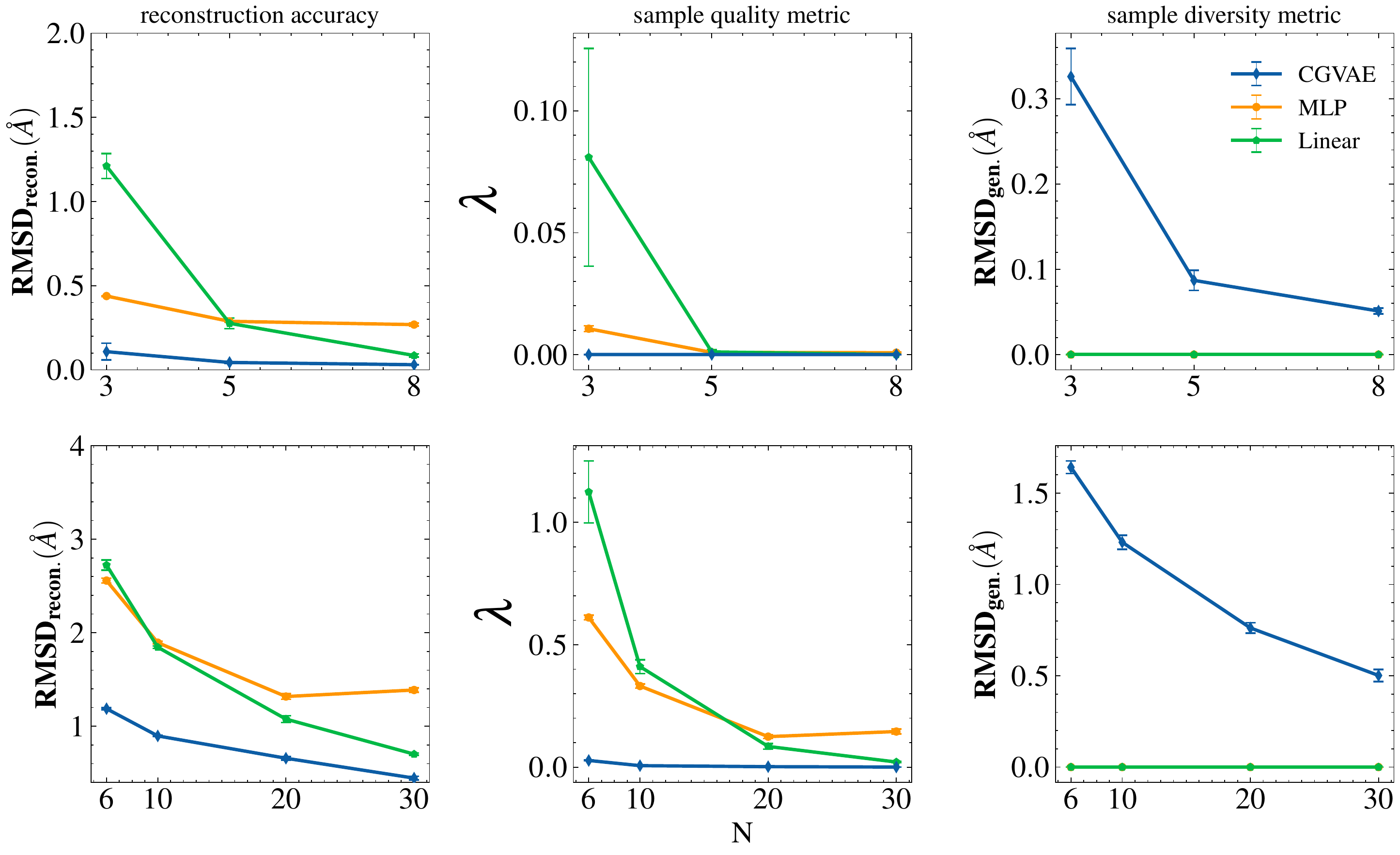}\label{fig:quant_results_f2}}
  \vspace{-2pt}
  \caption{\textbf{Left}: Examples of CG and generated FG geometries of alanine dipeptide (top) and chignolin (bottom) generated by CGVAE. The geometry visualizations are created by PyMol~\citep{pymol}. \textbf{Right}: Benchmarks of reconstructed and generated geometries for heavy-atom structures with different resolution for alanine dipeptide (top) and chignolin (bottom). 
%   $\rmsdgen$ are zeros for baselines because the generation process is deterministic.
  }
%   \vspace{-10pt}
  \label{fig:quant_results}
\end{figure*}

In this section, we introduce our experimental settings and discuss both the quantitative and qualitative results. We first introduce the benchmark datasets and comprehensive metrics in \cref{subsec:exp:benchmark} and \cref{subsec:exp:metric}, and then visualize and analyze the results in \cref{subsec:exp:results}.
% Inspired by benchmark metrics proposed for generative modeling molecular graphs and geometries \citep{polykovskiy2020molecular, xu2021learning}, we propose metrics to test the performance of the model on sample quality and diversity. 

\subsection{Benchmarks}
\label{subsec:exp:benchmark}
\subsecvspace
\textbf{Datasets.} We evaluate our models on 2 datasets that have been used as benchmarks for data-driven CG modeling ({\citep{wang2019coarse, wang2019machine, husic2020coarse}}): MD trajectories of alanine dipeptide and chignolin. The details of the simulation protocol can be found in \citet{husic2020coarse}. Alanine dipeptide is a classical benchmark system for developing enhanced sampling methods for molecular conformation, containing 22 atoms. Chignolin is a 10-amino acid mini-protein with 175 atoms and features a $\beta$-hairpin turn in its folded state. It is a more challenging structure to evaluate our model's capacity of decoding complex conformations.
For alanine dipeptide and chignolin, we randomly select 20,000 and 10,000 conformations respectively from the data trajectories and perform 5-fold cross-validation to compute mean and standard errors of evaluation metrics.

\textbf{Baselines.} 
We compare our model to two baseline models proposed in recent works. The first method is a learnable \textbf{linear} backmap matrix \citep{wang2019coarse}. The method itself is $O(3)$ equivariant, and we fix the translation freedom by re-centering the geometry to its geometric center for each sample. This method is a data-driven and equivariant extension of the approach described in \citet{wassenaar2014going}, which uses linear and deterministic (not learnable) projection followed by a random displacement. The second baseline employs Multi-Layer Perceptrons (\textbf{MLP})~\citep{an2020machine} to learn the projection. The model transforms the CG coordinates input and the FG coordinates output as 1D flattened arrays, shaped $3N$ and $3n$ respectively. The method uses neural networks to learn the mapping of the flattened vectors, which is not equivariant. Note that, both baselines lack the probabilistic formulation and are deterministic, which suffer from poor sampling diversity.
% their $\rmsdgen$ are zeros.

% \textbf{Evaluation protocols} To demonstrate that CGVAE is mapping-agnostic, we generate $\proj$ with random protocols for both datasets. For alanine dipeptide training, each $\map(i)$ is randomly selected from $[N]$. For the training of chignolin MD data, we developed a mapping generation protocol to ensure closer atoms are coarse-grained together based on the backbone topology of the protein. See Appendix \cref{appendix:eval_metrics} for more details. We perform 5-fold cross-validation for all the experiments to compute means and standard errors. For each fold, we use a randomly generated CG mapping.

\textbf{CG mapping generations.} Our method is compatible with arbitrary mapping protocols. For training our CGVAE and baseline models, in addition to the FG structures provided in the datasets, we also need a CG mapping to produce the CG coordinates. In this paper, we use AutoGrain proposed in \cite{wang2019coarse}, which learns the CG assignment via the objective function that penalizes distant FG atoms being assigned to the same CG atoms. In practice, we first train the AutoGrain model on each dataset, and then take the learned assignment to compute the CG coordinates $X$ from FG coordinates $x$. More details about CG mapping generation can be found in \cref{appendix:map_protocl}.

\begin{figure*}[!thp]
\centering
  \centering
  \vspace{-10pt}
  \subfloat{\includegraphics[width=0.43\textwidth]{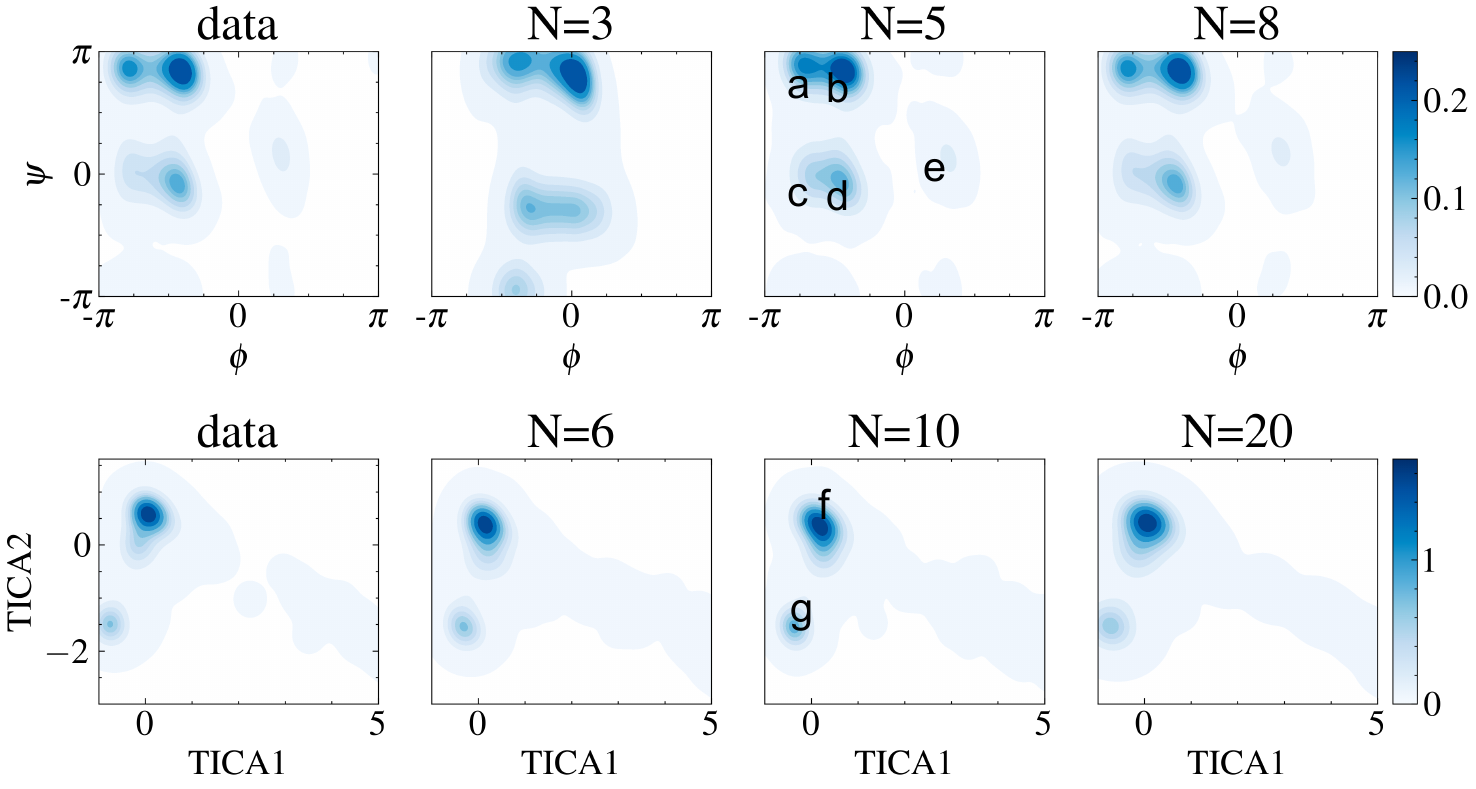}\label{fig:f1}}
  \hfill
  \subfloat{\includegraphics[width=0.57\textwidth]{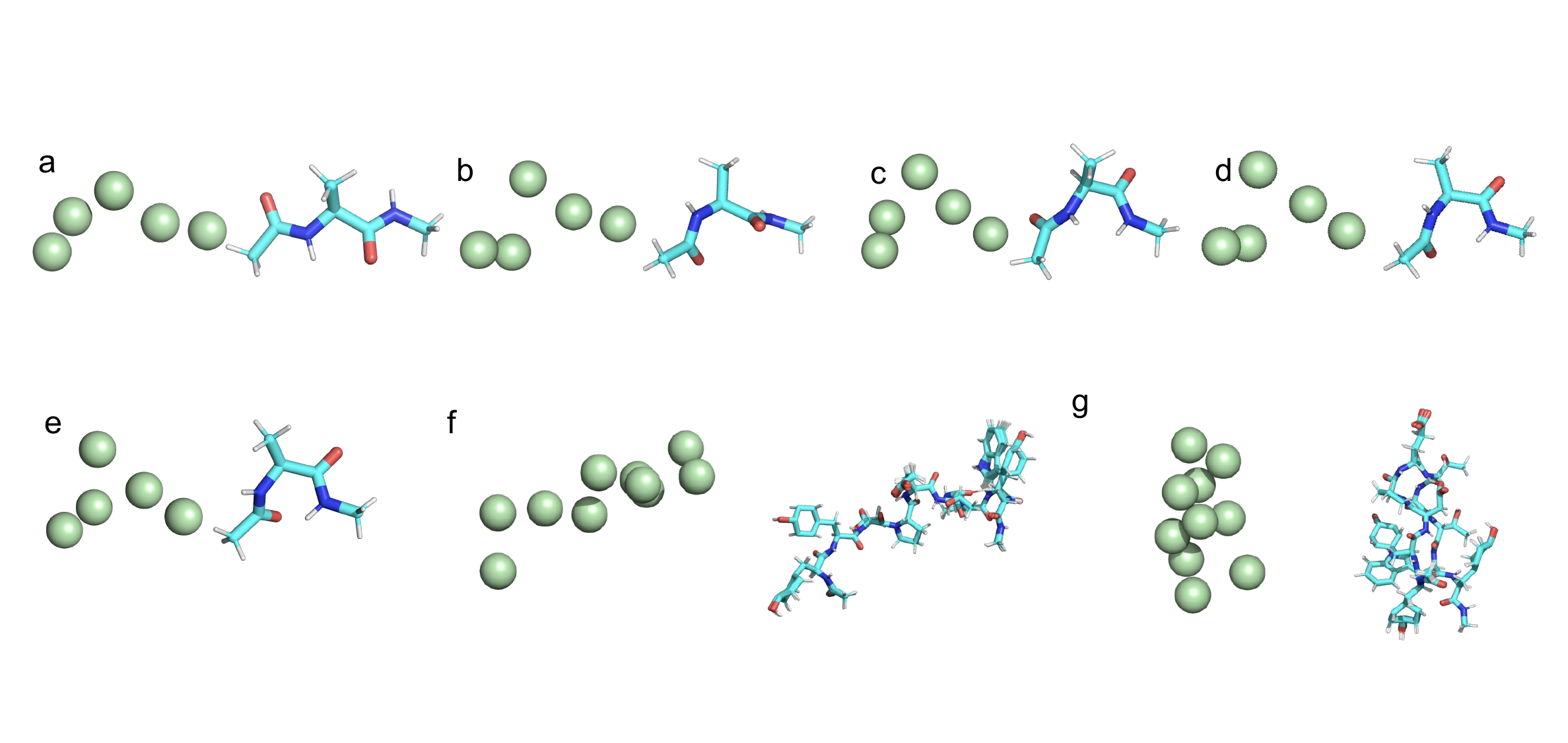}\label{fig:f2}}
  \vspace{-2pt}
\caption{ 
% \textbf{Top:} Ramachandran plots of conformations in the data and sampled geometries from CGVAE trained with different $N$. The generated samples recovers the important meta-stable states. The plot from higher resolution model shows better resemblance with the ground truth data. \textbf{Bottom:} 2D TICA plots of chignolin MD data and generated sampled from CGVAE trained at different $N$. The recovered TICA plots agree with the data. 
\textbf{Left:} 2D Ramachandran plots (top) and TICA plots (bottom) for conformations of alanine dipeptide and chignolin respectively. The structures for the plots are references in the dataset (first column) and samples generated by CGVAE trained with different $N$ (the rest three columns). Figures show that the generated samples recover the important meta-stable states and the distributions agree well with the ground truth. It also shows that the higher resolution model shows better resemblance with the ground truth data. \textbf{Right:} Visualization of representative samples in the 2D plots.
}
%  \vspace{-10pt}
\label{fig:chig_tica}
\end{figure*}

\subsection{Metrics}
\label{subsec:exp:metric}
\subsecvspace
Ideally, the generated FG samples should match the original FG molecular structures (quality) and consist of novel conformations (diversity). We propose  evaluation metrics for generation and reconstruction tasks that measure model fitting capacity, sample quality, and sample diversity at different CG resolutions $N$. More detailed descriptions of the proposed evaluation metrics are in \cref{appendix:eval_metrics}. We evaluate the reconstructed and sampled 3D geometries for both all-atom and heavy-atom (excluding hydrogen atoms) structures because heavy-atom encodes more important structural dynamics.

\textbf{Reconstruction accuracy.} 
% We use the reconstruction root mean squared distance ($\textnormal{RMSD}_{\textnormal{recon}}$). 
The reconstruction task evaluates the model's capacity to encode and reconstruct coordinates. Given the reference and reconstructed FG structures, we report their root mean squared distance ($\textnormal{RMSD}_{\textnormal{recon}}$).  A lower $\textnormal{RMSD}_{\textnormal{recon}}$ indicates the more powerful learning capacity of the model for reconstructing FG geometries. 
% \vspace{-0.5em}

\textbf{Sample quality.} We evaluate the sample quality by measuring how well the generated geometries preserve the original chemical bond graph.  Specifically, we report the graph edit distance ratio $\lambda(\graph_{\textnormal{gen}}, \graph_{\textnormal{mol}})$, which measures the relative difference between the connectivity of the generated geometries $\graph_{\textnormal{gen}}$ and the ground truth bond graph $\graph_{\textnormal{mol}}$.
For each generated conformation, $\graph_{\textnormal{gen}}$ is deduced from the coordinates by connecting bonds between pair-wise atoms where the distances are within a threshold defined by atomic covalent radius cutoff (see \cref{app:graph_struct}). The lower $\lambda(\graph_{\textnormal{gen}}, \graph_{\textnormal{mol}})$ is, the better the $\graph_{\textnormal{gen}}$ resembles $\graph_{\textnormal{mol}}$, with $\lambda = 0$ indicating $\graph_{\textnormal{mol}}$ and $\graph_{\textnormal{gen}}$ are isomorphic. 

\textbf{Sample diversity.} This metric measures the model capacity to explore the large conformation space. To evaluate sample diversity, we repeatedly sample multiple $z \sim P(z|X)$ and generate an ensemble of geometries using $\dec(z, X)$. We compare chemically valid generated geometries with the reference FG geometries using $\textnormal{RMSD}$, and term this metric $\rmsdgen$. Similar ideas have been adopted in related conformation generation literature~\citep{mansimov2019molecular}. A higher value of $\rmsdgen$ indicates the better diversity of the geometries generated by the model. Note that, since both baselines are deterministic without the capacity to model sampling diversity, their $\rmsdgen$ are zeros.

\subsection{Results and Analysis}
\label{subsec:exp:results}
\subsecvspace

\textbf{Quantitative results.} We test our CGVAE models against baseline methods on both datasets. A summary of all the metrics regarding the heavy-atom geometries is reported in \cref{fig:quant_results}. Results for all-atom geometries are left in the \cref{app:tables}. As shown in \cref{fig:quant_results}, CGVAE models can consistently achieve better reconstruction accuracy, sample quality, and diversity than linear and MLP models among all resolutions. The linear model ensures equivariance but lacks the expressive capacity to encode complex conformations. The MLP model enjoys higher expressiveness, but the lack of equivariance limits its generalization to different roto-translational states. Combining equivariance and MPNNs, CGVAEs can generate FG samples with high quality and diversity. 
Moreover, we observe that higher-resolution CG representations lead to better reconstruction accuracy, as channels for each CG node are responsible for fewer atoms. Besides, the higher $\rmsdgen$ results of the lower resolution models indicate that they can generate more diverse samples because the distribution of $z$ encodes a larger conformational space. Surprisingly we notice that the complex geometries can be well represented with a very small set of CG atoms. Our model works well even for the extremely coarse order reduction of 22-to-3 and 175-to-6 for alanine dipeptide and chignolin respectively. Such extent of CG reduction means potential acceleration of all-atom simulations by one or two orders of magnitude if a $\mathcal{O}(n \log n)$ algorithm such as Particle Mesh Ewald (PME) is used \citep{darden1993particle}.       

\textbf{Analyzing generated conformations.} To gain a better qualitative understanding of the generated samples, we project the generated samples onto lower-dimensional reaction coordinates for visualization. The visualizations are shown in \cref{fig:chig_tica}. For alanine dipeptide samples, we visualize the histogram of the generated samples on $\phi \times \psi$ which are two special central dihedral angles. The resulting histogram is called the Ramachandran plot, which is a classic visualization for analyzing peptide conformations \citep{ramachandran1963stereochemistry}. The distribution of these two dihedral angles is an important observable because the heavy-atom dynamics of the central alanine are completely described by $\phi$ and $\psi$ \citep{feig2008alanine}. As shown in the plot, with only 3 CG atoms, the model can nearly recover all the important meta-stable states, even though the dihedral angles can only be fully described by at least 6 heavy atoms at the center of the backbone. For analysis of generated chignolin samples, we perform Time-structure Independent Components Analysis (TICA), a popular analysis method to project dynamics onto slow degrees of freedom \citep{perez2013identification, scherer_pyemma_2015}. It is a useful analysis to project high-dimensional complex conformation variations onto a 2D space to identify meta-stable states of molecular geometries. To conduct the analysis, we use pair distances for backbone atoms (45 atoms, 990 distances) with a lag time of 10 ns for both the ground-truth data and generated samples. We first parameterize the TICA transformation on the simulation data by PyEmma \citep{scherer_pyemma_2015}, and apply the transformation on the generated geometry sets to produce a plot for comparison. As shown in the figure, our TICA density plot made from generated samples agrees well with the reference data, which demonstrates our model accurately captures the distribution of conformational states.

\section{Conclusion}
\secvspace
In this work, we present the Coarse-Grained Variational Auto-Encoder (CGVAE), a principled probabilistic model to perform stochastic reverse generation to recover all-atom coordinates from coarse-grained molecular coordinates. Built on equivariant operations, our model encodes all-atom information into invariant latent variables which are sampled to generate coordinates equivariantly. Comprehensive experiments demonstrate that CGVAE significantly outperforms existing methods. For future work, we plan to extend our general framework to the more complex condensed phase systems which demand accurate backmapping at both intra-molecular and inter-molecular level. We hope that our efforts further improve the predictive power of CG simulations for accelerated chemical and materials discovery.

\section*{Code Availability}
A reference implementation can be found at \href{url}{https://github.com/wwang2/CoarseGrainingVAE}

\section*{Acknowledgement}
This work is supported by Toyota Research Institute and Google Faculty Award.

\newpage
\bibliographystyle{icml22/icml2022}
\bibliography{ref}
% \end{document}

\newpage
\appendix
% \section{Appendix}
\onecolumn
\label{appendix}
\begin{table}[!htp]
\small
\caption{Summary of important notations.} 
\begin{center}
% \resizebox{\textwidth}{!}
{
\begin{tabular}{@{}l|l@{}}
    \hline
    \toprule
    Symbol & Meaning \\
    \midrule
    % \midrule
    \hline
    & Preliminaries \\ \hline
    $x\in \mathbb{R}^{n \times 3}$ and $X\in \mathbb{R}^{   N \times 3}$ & FG geometry and CG geometry \\
    $i/I$ & index over atoms/beads in FG/CG geometry \\
    $\proj \in \mathbb{R}^{N \times n}$ & CG projection operator \\
    $\lift\in \mathbb{R}^{n \times N}$ & lift operator. $\lift \proj = \Pi_n$, $\proj \lift  = I_N$. \\
    $Q \in \mathbb{R}^{3 \times 3}$ & 3 $\times $ 3 orthogonal transformation $Qx = \{Q x_1, ... Q x_n\}$\\ 
    $g\in \mathbb{R}^{3}$ & translational vector. $x+g = \{x_1 + g, ..., x_n + g\}$ \\
    $\map$ & assignment map from $[n]$ to $[N]$ \\
    $w_i$ & weights for atom $i$ for CG projection \\
    $C_I$ & $C_I = (k\in[n]|\map(k)=I)$\\ 
    \hline

    & Encoder and prior \\ \hline
    $\enc$  or $\enc$(x,X) & encoder in VAE formulation. In CGVAE, it is instantiated as $q_{\phi}(z|x, X)$ \\
    $\prior$ or $\prior(X)$ & prior in VAE formulation. In CGVAE, it is instantiated as $p_{\psi}(z|X)$. \\ 
    $\msg$ &  message function \\
    % $\emsg$ & equivariant message \\
    $\update$ & update module \\ 
    $\textnormal{RBF} : \R \rightarrow \R^K$ & radial basis function following \citet{schutt2021equivariant, klicpera2020directional} \\
    $h^t_i/H^t_I \in \mathbb{R}^{F}$ & FG/CG invariant embeddings at $t_{th}$ convolution step in $\enc$ or $\prior$ model \\
    $\widetilde{H}_{I}^t$ &  CG node embeddings updates for node $I$ pooled from FG information in \cref{eq:cg2cg}\\
    $\Delta H_I^t, \Delta V_I^t, \Delta \widebar{H}_I^t, \Delta \widebar{V}_I^t$ & residual udpates for scalar, vector, pseudoscalar, and psuedovector at convolution depth $t$ \\
    $N(\cdot) $ & the neighbor function to obtain adjacent node in a graph\\
    $F/K$ & dimension for node/edge features \\ 
    $\mathcal{N}(\cdot | \cdot, \cdot)$ & multivariate gaussian distribution\\
    $z\in \mathbb{R}^{N \times F}$  & latent representation on CG beads \\
    $\epsilon$ & reparameterized Gaussian noise \\
    $\encinvout \in \mathbb{R}^{N \times F }$ & $\encinvout :=\enc(x, X) =p_{\phi}(z|x, X)$ which is the product of CG-wise normal distribution.\\ 
    $\priorinvout \in \mathbb{R}^{N \times F }$ &  $\priorinvout:=\prior(X)=p_{\psi}(z|X)$ which is the product of CG-wise normal distribution.  \\ 
    \hline
    & Decoder \\ \hline
    $W_{1}-W_{9}$ & parameters for equivariant update in \cref{eq:dec}  \\
    $L_1, L_2$ & MLPs used to parameterize $W$ in \cref{eq:dec}. $L_1: \R^F \rightarrow \R^F$, $L_2: \R^K \rightarrow \R^F$ \\
    $\ind(i, C_I)$ & index function. $\ind(i, C_I)$ retrieves the index of $i \in C_I$. 
    For example, $\ind(1, (1,2,4)) = 0$ \\
    $\decoder(\cdot)$ & a generic decoder function : $\R^{N\times 3} \rightarrow \R^{n \times 3}$ \\
    $\dec$ or $\dec(X, z)$  & decoder in VAE formulation. In CGVAE, it is instantiated as $p_{\theta}(x|X, z)$. \\
    $\tx$ & $\mathbb{R}^{n \times 3}$, decoded FG geometry. $\tx=p_{\theta}(x|X, z)$. \\    
    $\dtx$/$\dtxi$ & $\dtx := \tx  - X$. $\dtx_i$ is $i_{th}$ coordinate of $\dtx$. Note that decoder in CGVAE generates $\dtx$. \\    
    $H$, $V$, $\widebar{H}$, $\widebar{V}$ & scalar, vector, pseudoscalar, pseudovector features  \\
    $\decequiout \in \mathbb{R}^{N \times F \times 3}$ & equivariant output of decoder, subsets of its vector channels are used as prediction for $\dtx$ \\     
    % $X\in \mathbb{R}^{N \times 3}$ & CG geometry \\
    % $\dtxi$ & $\dtx_i$, coordinate difference between original FG geometry $x$ and generated FG geometry $\tx$ on i-th beads \\ 
    $\eupdate(\cdot, \cdot)$ & equivariant update function proposed in \citet{schutt2021equivariant} \\
    \hline
    & Misc \\ \hline
    $\dcut/\Dcut$ & cutoff distance to form a graph on FG/CG geometry \\    
    $\{\graph, \node, \edge\}$ & graphs, nodes, and edges \\
    $d_{ij}/d_{IJ}\in \R$ & inter-node distances in FG/CG level. \\
    $d_{iI} $ & distance between $i_{th}$ FG atom and $I_{th}$ CG atom \\
    $\widehat{E}_{IJ}=\frac{X_J-X_I}{d_{IJ}}$ & unit distance vector between I and J\\

    \hline
    & Experiments \\ \hline
    $\graph_{\textnormal{mol}}$ & molecular graphs based on chemical bond connectivities \\ 
    $\graph_{\textnormal{sample}}$ & deduced molecular graphs of samples generated by CGVAE \\ 
    $\edge$ & edge set that includes 2-hop neighbors in $\graph_{\textnormal{mol}}$ \\
    $\lambda(\cdot, \cdot)$ & graph edit distance ratio to measure sample quality by comparing $\graph_{\textnormal{mol}}$ and $\graph_{\textnormal{sample}}$\\
    $\textnormal{RMSD}_{\textnormal{recon}}$ &measure reconstruction loss; the lower, the better \\
    $\textnormal{RMSD}_{\textnormal{gen}}$ & measure diversity of generated samples; the higher, the better.  \\
    \hline 
    & Operators \\ \hline
    $\circ$ & pointwise product \\
    $[\cdot, \cdot]$ & concatenation\\
    $\times$ & cross product \\
      \bottomrule
\end{tabular}
}
\end{center}
\label{table:symbol_notation}
\end{table}
\section{Symbol Table}
We summarize definitions of important notations in \cref{table:symbol_notation}.

\section{Background and Related Works}

\subsection{MD Background}
\textbf{Molecular dynamics and CG.} 
The dynamics and distribution of $x$ are governed by a Hamiltonian $u(x)$. Under the ergodic hypothesis, the distribution of $x$ converges to the Boltzmann distribution: $p(x) \propto e^{-u(x) \beta} $ where $\beta$ is the inverse temperature of the system heat reservoir. Given $\map$ and $p(x)$, the equilibrium distribution of $X$ can be derived by $\int p(x) p(X|x) \: dx $ where $ p(X|x)$ is the dirac delta function $\delta(X - \proj(x))$ which map $x$ deterministically onto the space of $X$. In statistical mechanics, $-\frac{1}{\beta} \ln p(X)$ is also defined as the Potential of Mean Force (PMF) which is a conservative free energy quantity to describe the equilibrium distributions of coarse-grained variables. In this work we are interested in solving the inverse problem of learning the generative map $p(x | X )$ given some mapping $\proj (x) = X$.

\textbf{CG simulation.} The theory and methodology of parameterizing CG models have been well established over decades~\citep{kirkwood1935statistical, zwanzig2001nonequilibrium, noid2008multiscale}. Systematic parameterization and mapping protocols have been widely applied to simulating large-scale biological molecules \citep{marrink2007martini, menichetti2019drug, souza2021martini}. Lately, there are also a handful of works utilizing machine-learned potentials to simulate CG systems \citep{zhang2018deepcg, wang2019coarse, wang2019machine, husic2020coarse}.

\subsection{More Related Works}

\textbf{Equivariant modeling of point sets.}
Efforts have been made to build equivariant neural networks for different types of data and symmetry groups \citep{zaheer2017deep, qi2017pointnet, maron2018invariant, thomas2018tensor, weiler20183d,cohen2019gauge, worrall2019deep, weiler2019general,cohen2019gauge, fuchs2020se, finzi2020generalizing, maron2020learning, bevilacqua2021equivariant} . In the case of 3D roto-translations for point clouds, tensor field network (TFN) \citep{thomas2018tensor} and its extension with attention mechanism $SE(3)$ transformer \citep{fuchs2020se} have been applied to various domain where modeling roto-translation is crucial \citep{winkels20183d, unke2021se, batzner2021se, cai2021equivariant}. 

\textbf{Generative model for molecular conformations.}
Many generative models have been proposed to generate molecular conformation ensembles conditioned on molecular graphs with modeling techniques like autoregressive models \citep{gebauer2019symmetry}, conditional normalizing flow \citep{xu2021learning}, bilevel programming \citep{xu2021end}, and gradient-based optimizations \citep{shi2021learning}.

\section{Mathematical Details}
\subsection{Lift Operator}
\label{app:lift}
\begin{restatable}[]{definition}{lift_op}
\label{def:lift}
A lift operator $\lift$ of size $n \times N$ that maps the point sets space from $\R^{N \times 3}$ to $\R^{n \times 3}$ satisfies the following property. 
$ %\begin{equation*}
\lift_{i, I}=\left\{\begin{array}{ll}
1 & \text { if } I = \map\left(i \right) \\
0 & \text { otherwise }
\end{array}\right.$

\end{restatable}
% \begin{definition}
% \end{definition}
In the case of $w_i =1$ ($w_i$ is the projection weight defined in \cref{section:cg_def}) for all $i$, the $\lift$ is the Moore-Penrose pseudoinverse of $M$ that has the following properties \citep{loukas2019graph, cai2020graph}.

\begin{restatable}[]{property}{composition}
\label{prop:composition}
$\proj \lift = I_N. $ %\textnormal{ When} f_{i}=1 \textnormal{ } \forall i \in [n], \lift \proj = \Pi_{n}. $
\end{restatable}
\begin{proof}
First note that $\lift_{i, I} = \mathrm{I}_{I \in m(i)}$ where $ \mathrm{I}$ is the indicator function. Expanding the definition, we have % Because each $i$ only gets mapped to an unique $I$ and $\proj_{I, i}$ is normalized (\cref{section:cg_def}), we have:
\begin{equation*}
    \sum_{i \in [n]} \proj_{I, i} \lift_{i, J} 
    =\sum_{i\in [n]} \frac{w_i}{\sum_{j \in C_I} w_j} \mathrm{I}_{J \in m(i)}    
    =  \sum_{i\in [n]} \frac{w_i \mathrm{I}_{I \in m(i)}}{\sum_{j \in C_I} w_j} \delta_{IJ} 
    =  \frac{\sum_{i\in C_I} w_i}{\sum_{j \in C_I} w_j} \delta_{IJ} =  \delta_{IJ}
\end{equation*}
\end{proof}

\subsection{A Toy Example}
\label{app: toy-exampe}
Consider a cluster map $\map$ that atoms $\{1, 2, 3, 4\}$ into 2 groups $\{\{1, 2, 3\}, \{4\}\}$. Assuming the mapping is constructed as geometric centers, i.e., the mapping weights are equal for each CG particle, then
$   \proj = 
\begin{bmatrix}
1/3 & 1/3 & 1/3 & 0 \\
0 & 0 & 0 & 1 
\end{bmatrix} 
$
and 
$
    \lift = \begin{bmatrix}
1 & 0 \\
1 & 0 \\
1 & 0 \\
0 & 1 
\end{bmatrix}
$. 
It naturally follows that  $ % \begin{equation*}
    \proj \lift = \begin{bmatrix}
1 & 0 \\
0 & 1 
\end{bmatrix}  
$ % \end{equation*}
and
$ %\begin{equation*}
    \lift \proj = \Pi_{n} = \begin{bmatrix}
1/3 & 1/3 & 1/3 & 0\\
1/3 & 1/3 & 1/3 & 0\\
1/3 & 1/3 & 1/3 & 0\\
0 & 0 & 0 & 1\\
\end{bmatrix}.  
$ %\end{equation*}

 Note that the normalization condition $\sum_i \proj[I, i] = 1$ is essential for $\proj$ to respect translational equivariance because CG needs to preserve the norm of the translation vector $g$.

Note that the construction of CG projections do \emph{not} require that $\proj_{I, i}  \neq  0 \; \textnormal{ } \forall i \in C_{I} = \{k \in [n] | \map(k) = I\} $ because $w_i$ can take arbitrary non-negative values. Intuitively, this means that some atoms do not contribute to the construction of CG coordinates, but they are assigned to a bead via $\map$ in $C_I$. This can sometimes be seen in some coarse-grained simulations in practice: for instance protein simulations are often run with $\alpha$ carbons only. In the case where $w_i =  0$ for some $i$, $\lift$ the operator defined is no longer the pseudo-inverse of $\proj$. For example, if we assign 0 weights to two of the nodes in $\proj$ induced by $\map$ is
$ \proj = \begin{bmatrix}
0 & 1 & 0 & 0 \\
0 & 0 & 0 & 1 
\end{bmatrix}
$.   
Based on \cref{def:lift}, because $\lift$ is constructed from $\map$, $\lift$ remains the same but is not the pseudo-inverse of $\proj$ in this case.

\subsection{Proof of Property \ref{property:mequivariance}}
\label{app:m-equivariance}

\mequivariance*
\begin{proof}
  % For each $I$, coarse coordinates are generated by: $X_I = \sum_i \proj_{I, i} x_i$. When applying an orthogonal and translational transformation to $x_i$:
  Without loss of generality, we prove $f_M(Qx+g) = Qf_M(x) + g$ holds for any bead $I$. 
\begin{equation*}
    [f_M(Qx+g)]_I=\sum_i \proj_{I, i} (Q x_i + g) = \sum_i \proj_{I, i} (Q x_i) + \sum_i\proj_{I, i}g = QX_I + \sum_i\proj_{I, i}g = QX_I + g = [Qf_M(x) + g]_I
\end{equation*}
The first equality follows from the definition of $f_M$. The second equality follows from the linearity of $M$. The third equality definition of $Q$ and $M$. The fourth equality follows from the definition of $M$. The last equality follows from the definition of $f_M$.
% Because $Q$ operates in $\R^3$ and is linear, so it can be taken out of the summation. Further, the normalization condition for $\proj_{I, i}$, implies $\sum_i \proj_{I, i} g = g$. Therefore, we have: 
% \begin{equation*}
%     \sum_i \proj_{I, i} (Q x_i + g) = Q (\sum_i \proj_{I, i} x_i) + g =  Q X_I + g 
% \end{equation*}
% Following from the definition of $E(3)$ equivariance \cite{satorras2021n}, we have shown that $\proj$ is equivariant.  
\end{proof}

\subsection{On the geometric Requirements for Backmapping}
\label{app:consistency}
Given the fact that $\proj$ is $E(3)$ equivariant, we can show that $\decoder$ has to be equivariant for $\reqone$ to be true.

\begin{proposition}
    If $\proj \decoder(X) = X$ and $\proj$ is $E(3)$ equivariant, $\decoder$ is also $E(3)$ equivariant, i.e. \decoder(QX + g) = Q \decoder(X) + g.
\end{proposition}

\begin{proof}
We show $\decoder(Qx + g) = Q \decoder(X)+g$ by proof of contradiction. Assume $\decoder(QX + g) \neq Q\decoder(X) + g$. Given $\proj \decoder(X) = X$ and $\proj$ is $E(3)$ equivariant, we have  
\begin{equation}
    \proj \decoder(QX + g) \neq \proj Q\decoder(X) + Mg = \proj Q\decoder(X) + g = Q \proj \decoder(X) + g = QX + g
\end{equation}
The first inequality follows from the assumption. The second equality follows from the definition of $M$. The third equality follows from the property of $M$ and $Q$. The last equality follows from the property of $\decoder$.
We reaches a contradiction that $\proj \decoder(QX + g) \neq QX + g$. Therefore we conclude that the assumption has to be false and $\decoder(QX + g) = Q \decoder(X)+g$.

\end{proof}

\subsection{ELBO Derivation}

Following the probabilistic graphical model in \cref{fig:proba_graph}, we carry out the standard derivation of the Evidence Lower Bound (ELBO) described in \cref{eq:elbo} for our conditional VAE here \citep{sohn2015learning}: 
\begin{equation*}
\begin{aligned}
    \log p(x|X) &= \log \E_{ q_\phi(z | x, X) } \frac{ p(x| X) }{q_\phi(z | x, X)} \\
    &\geq \E_{ q_\phi(z | x, X) } \log \frac{ p(x|X) }{q_\phi(z | x, X)} \\
    &= \E_{q_\phi(z | x, X)} \log \frac{  p_\theta(x | X, z)p_\psi(z|X)}{q_\phi(z | x, X)} \\
    &= \underbrace{\E_{q_\phi(z | x, X)} \log p_\theta(x | X, z)}_{\gL_{\text{recon.}}} + \underbrace{\E_{q_\phi(z | x, X)} \log \frac{ p_\psi(z|X) }{q_\phi(z| x, X)}}_{\gL_{\text{reg.}}}
% \label{eq:elbo}
\end{aligned}
\end{equation*}

%\begin{figure}
\begin{wrapfigure}{R}{0.2\textwidth}
% \inlinefigtopvspace % \vspace{-35pt}
\centering
\includegraphics[width=\linewidth]{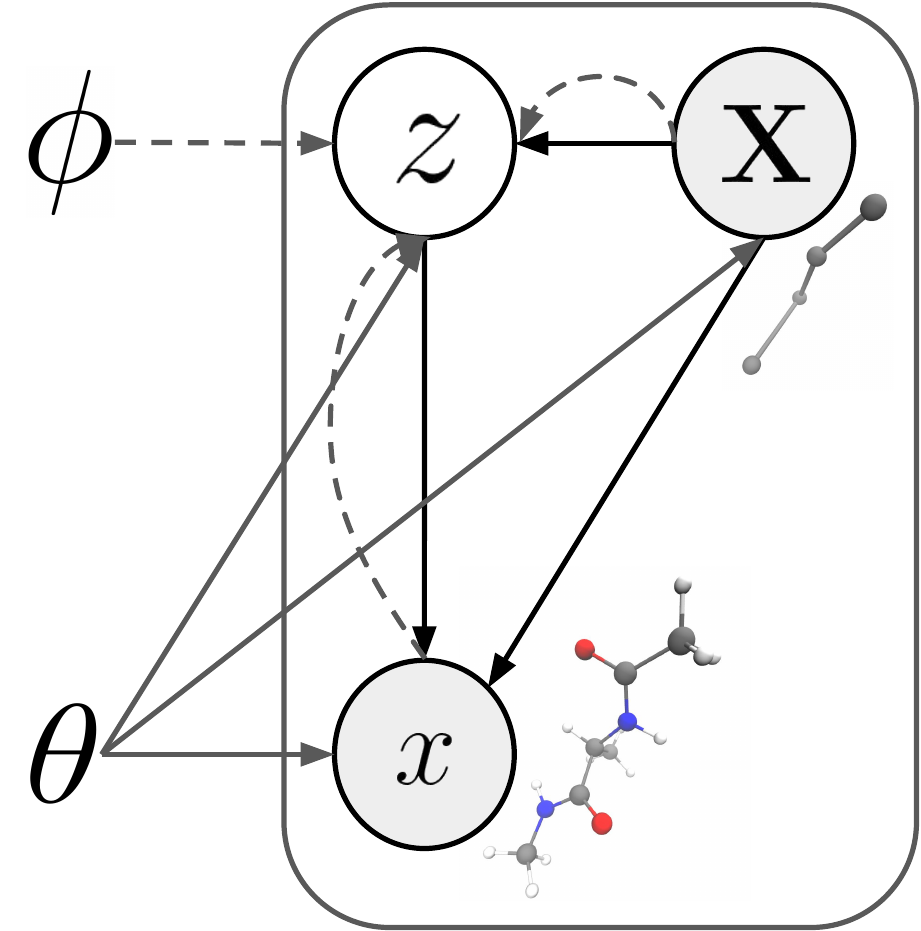}
% \inlinefigbottomvspace % \vspace{-10pt}
\caption{The probabilistic graphical model for the generation process (solid line) and inference process (dashed line) of the conditional geometric reverse mapping.}
\label{fig:proba_graph}
%\vspace{-120pt}
%\end{figure}
\end{wrapfigure}

\section{Graph Representations}
\label{app:graph_struct}
\begin{table}[!htp]\centering
\tiny
\renewcommand{\arraystretch}{1.3}
\caption{Summary of graph data inputs used in CGVAE.}\label{tab:graph_struct}
\begin{tabular}{@{}llll@{}}
\toprule
&$\graph_{\textnormal{FG}}$ &$\graph_{\textnormal{CG}}$ \\\midrule
Coordinate &$x \in \R^{n \times 3}$ &$X \in \R^{N \times 3}$ \\
Node &$\node_{\textnormal{FG}}$ &$\node_{\textnormal{CG}}$ \\
Edge &$\edge_{\textnormal{FG}} = \{(i, j) \mid d(x_i, x_j) < \dcut, \forall \; i,j \in \node_{\textnormal{FG}} \} $ &$\edge_{\textnormal{CG}} = \{ (I, J) \mid (C_{I} \times C_J) \cap \edge_{\textnormal{FG}} \neq \ \emptyset, \forall I, J \in \node_{\textnormal{CG}} \}$ \\
Node feature & invariant features $h \in \R^{n \times F}$  &invariant features $H \in \R^{N \times F}$, equivariant features $V \in \R^{N \times F \times 3 }$ \\
Edge feature &$d_{ij} \in \R$ &$d_{IJ} \in \R$, $\widehat{E}_{IJ} = \frac{X_I - X_J }{d_{IJ}} \in \R^{3}$ \\
\bottomrule
\end{tabular}
\label{tab:grpah}
\end{table}

Here we describe the graph data structure we used for CG and FG graphs for geometric convolutions. The detailed description is summarized in \cref{tab:graph_struct}. 

\textbf{FG 3D graph $\graph_{\textnormal{FG}}$.} The geometric graph at the FG level is represented as an undirected $\graph_{\textnormal{FG}} = (\node_{\textnormal{FG}}, \edge_{\textnormal{FG}})$, where $\node_{\textnormal{FG}} = [n]$ is the node set. The edge set is determined by a predefined distance cutoff $\dcut$ that $\edge_{FG} = \{(i, j) \mid d(x_i, x_j) < \dcut, \forall \; i,j \in \node \}$, where $d(a, b)=\|a - b\|_2$ is the euclidean distance between $a$ and $b$. Following standard featurization procedure in machine learning techniques for molecular graph, the initial node features $h \in \R^{F}$ is simply one-hot encoding of atomic types. Following \cite{schutt2021equivariant}, we parameterize edge features using $\rbf(d_{ij})$ to expand $d_{ij}$ with coefficients of $K$ sine radial basis: $\sin(\frac{n \pi}{\dcut} d_{ij}) / d_{ij}$ with a cosine filter as a continuous cutoff. $n$ ranges from 1 to $K$. Both $K$ and $\dcut$ are hyper-parameters of the model. 

\textbf{CG graph.}
Given $\map$, we introduce a corresponding undirected CG graph $\graph_{\textnormal{CG}}$ with $\node_{CG} = [N]$. There is an edge between node $I$ and $J$ if there exists at least an edge between sub-nodes in $C_{I}$ and $C_J$: $\edge_{\textnormal{CG}} = \{ (I, J) \mid (C_{I} \times C_J) \cap \edge_{\textnormal{FG}} \neq \ \emptyset, \forall I, J \in \node_{\textnormal{CG}} \}$. Here, $C_{I} \times C_J$  indicates the Cartesian product between $C_{I}$ and $C_J$. The initial node feature $H_I$ is parameterized using the pooling operation described in \cref{section:encoder}. The distances for CG point sets are similarly featurized with sine radial basis functions, with a cutoff distance at $\Dcut$. For the equivariant decoding described in \cref{section:decoder}, CG graph also carries a set of equivariant features parameterized from linear combinations of $\widehat{E}_{IJ}$.

\textbf{FG bond graph.}
A molecule can have many conformations (stereo-isomers) which share the same bond graph structure. For classical MD simulations where there are no chemical reactions, the chemical bond graph does not change. The molecular graph $\graph_{\textnormal{mol}} = (\node_{\textnormal{mol}}, \edge_{\textnormal{mol}})$ can be deduced from chemical rules and can be obtained from open-sourced chemoinformatic analysis packages like \texttt{RDKIT} \citep{landrum2013rdkit}. 

\textbf{Determine bond graph from FG geometries.} For evaluation of quality of generated samples, we compute the bond graphs of generated samples and compare them with the ground truth chemical bond graph. The chemical graphs are essentially radius graphs but with pairwise distance cutoffs determined by the atom types of the pairs involved. The distance cutoffs for each pair of atoms are determined from the covalent radii of each atoms. We obtain  covalent radius of each atom from \texttt{RDKIT} using \texttt{rdkit.Chem.PeriodicTable.GetRCovalent}. Let $r_{\textnormal{covalent}}(i) $ indicates the covalent radius of atom $i$, the bond graph of generated geometries can be deduced: $\{ (i,j ) \mid  \|x_i - x_j\|_2 < 1.3 \times (r_{\textnormal{covalent}}(i) + r_{\textnormal{covalent}}(j)  \forall i, j \in \node_{\textnormal{FG}}\})$. The scaling factor of 1.3 is chosen based on the \texttt{xyz2mol} public repository \citep{xyz2mol}.  

\section{Model Details}
\label{appendix:model}
\subsection{Details of Encoder and Prior Networks}
\label{appendix:encoder_prior}

\textbf{Encoder.} We introduce the detailed design of our encoder. Recall that $K$ is the dimension of invariant edge features, and $F$ is the dimension of node embedding. We define $L_{3}, L'_{3}, L''_{3}, L'''_{3}: \mathbb{R}^K \rightarrow \mathbb{R}^F$ and $L_{4}, L'_{4}, L''_{4}, L'''_{4}: \mathbb{R}^F \rightarrow \mathbb{R}^F$ be layer-wise MLPs with swish function as activation. The detailed form of FG MPNN is:
\begin{equation*}
\begin{aligned}
      h_i^{t+1} &= \update\left( \sum_{j \in N(i)} \msg(h_i^{t}, h_j^{t}, \rbf(d_{ij})), \; h_i^{t}\right)  = h_i^{t} + \sum_{j \in N(i)} L_{3}(\rbf(d_{ij})) \circ L_{4}( h_{i}^t \circ h_{j}^t) 
\end{aligned}
\end{equation*}
The detailed form of the pooling operation to map information from FG to CG level is:
\begin{equation*}
\begin{aligned}
    \widetilde{H}_{I}^t &= \update\left( \sum_{i \in C_{I = \map(i)}}\msg(h_i^t,H_{I}^t,  \rbf(d_{iI})) , \; H_I^{t} \right) =  H_{I}^{t} + \sum_{i \in C_{I}}  L'_{3}(\rbf(d_{iI})) \circ L'_{4}( h_i^{t}) 
\end{aligned}
\end{equation*}
Similar to MPNNs at the FG resolution, the CG-level MPNN is:
\begin{equation*}
\begin{aligned}
      H_I^{t+1} &= \update\left( \sum_{J \in N(I)} \msg( \widetilde{H}_{I}^t, \widetilde{H}_{J}^t, \rbf(d_{IJ})), \;  H_{I}^t \right)  =H_{I}^t + \sum_{J \in N(I)} L''_{3}(\rbf(d_{IJ})) \circ L''_{4}( \widetilde{H}_{I}^t \circ\widetilde{H}_{J}^t) 
\end{aligned}
\end{equation*}
This form of MPNN is used for CG information propagation for both $\enc$ and $\prior$. 

\textbf{Prior.} For $\prior$, the detailed depth-wise form of the message passing operation is : 
\begin{equation*}
\begin{aligned}
      H_I^{t+1} &= \update\left( \sum_{J \in N(I)} \msg( H_{I}^t, H_{J}^{t}, \rbf(d_{IJ})), \;  H_{I}^t \right)  =H_{I}^t + \sum_{J \in N(I)} L'''_{3}(\rbf(d_{IJ})) \circ L'''_{4}( H_{I}^t \circ \widetilde{H}_{J}^t ) 
\end{aligned}
\end{equation*}

\subsection{Details of Decoder}
\label{appendix:decoder}
In this subsection, We first give details about \cref{eq:dec} and \cref{eq:update} in decoder and discuss the its equivariance property. We then discuss the geometric consistency of the decoder. Lastly, we give the implementation of channel selection operations.

\subsubsection{Brief Introduction to Pseudoscalars and Pseudovectors}
\label{app:pseudo}
Pseudoscalars and pseudovectors are like scalars and vectors, but undergoes an additional sign flip under reflection. We use $\widebar{\cdot}$ to indicate a quantity is pseudoscalar or psuedovector. Let $A, B, C \in \R^3$ be normal vectors. pseudoscalar $\widebar{H}$ and pseudovector $\widebar{V}$ can be constructed as follows:
$\widebar{H} = A \cdot (B \times C), \widebar{V} = A \times B$. Note that cross product  obeys the following identity under matrix transformations $(MA) \times (MB) = \det(M)(A \times B)$.  
Let $P$ be a $3\times 3$  reflection transformation with $\det(P) = -1$. It is easy to see that $(PA \times PB) = - P (A \times B)$.  Along with the simple property of $PA \cdot PB = P(A \cdot B)$, we have $P\widebar{H}$ and $P\widebar{V}$
\begin{equation*}
\begin{aligned}
    P\widebar{H} & = PA \cdot (PB \times PC) = PA \cdot (-P(B \times C)) = - A \cdot (B \times C) = - \widebar{H} \\
    P\widebar{V} & = PA \times PB = -P(A \times B) = -\widebar{V}
\end{aligned}
\end{equation*}
Next, we list different ways to generate scalars/vectors/pseudoscalars/pseudovectors. These are used as the basis in \cref{eq:dec}. The details of these constructions for arbitrary order tensors can be found in the documentation of the e3nn package \citep{e3nn}.

% To satisfy the sign flip rule under reflection, pseudoscalars can also be constructed by dotting a pseudovector and a vector; pseudovectors can be obtained from products of pseudoscalars and vectors, and products of scalars and pseudovectors. Similarly, scalars and vectors, which do not flip sign under reflections, can be constructed from pseudoscalars and pseudovectors. Here, we list the construction rules: 
\begin{itemize}
\item scalar: i) $\widebar{H} \circ \widebar{H}$ ii) $\widebar{V} \cdot \widebar{V}$ 
\item vector: i) $V \times \widebar{V}$ ii) $\widebar{H} \circ \widebar{V}$ 
\item pseudoscalar: i) $V \cdot \widebar{V}$ ii) $H \circ \widebar{H}$
\item pseudovector: i) $V \times V$ ii) $\widebar{V} \times \widebar{V}$ iii) $ \widebar{H} \circ V $ 
\end{itemize}
The output of \cref{eq:dec} are scalars, pseudoscalars, vectors, and pseudovectors. As the construction of each equation in \cref{eq:dec} ensures that the ``type'' of output matches with the ``type'' of the basis, we immediately get the equivariance guarantee for the outputs of \cref{eq:dec}.

\subsubsection{Details of Equivariant Update Layers}
\label{app:eupdate}
In this subsection, we provide details of equivariant update layers in \cref{eq:update} 
$H_I^{t+1}, V_I^{t+1} = \sigma(H_I^t + \Delta H_I^t, V_I^t + \Delta V_I^t)
$.
\cref{eq:update} updates an invariant scalar feature $H_I \in \R^F$ as the first argument and an equivariant vector feature $V_I \in \R^{F \times 3}$ as the second argument. Under the action of rotation  $Q$, $\sigma$ has to satisfy the following constraint: $[\sigma(H_I, QV_I)]_1 = [\sigma(H_I, V_I)]_1$ and $[\sigma(H_I, QV_I)]_2 = Q[\sigma(H_I, V_I)]_2$. We next introduce the detailed forms of $\sigma$.

Let $L_{5}, L'_{5}, L''_{5}: \R^{2F} \rightarrow \R^{F}$ be MLPs and $W_{\alpha},W_{\beta} \in \R ^{F\times F}$  be two learnable matrices. Denote the invariant input of $\sigma(\cdot, \cdot)$ as ${H'_{I}}^{t+1} := H_{I}^{t} + \Delta {H_{I}}  \in \R^{F}$ and the equivariant input as ${V'_{I}}^{t+1} := V_I^t + \Delta V_I \in \R^{F \times 3}$. The detailed form of $\sigma$ to the update $H^{t+1}$ and $V^{t+1}$ is 
\begin{equation}
\begin{aligned}
 H_{I}^{t+1} & = {H'_{I}}^{t+1}  + L_{5} ( \left[{H'_{I}}^{t+1}, \left\|W_{\beta} {V'_{I}}^{t+1} \right\|\right])  + L'_{5} (\left[{H'_{I}}^{t+1}, \left\| W_{\beta} {V'_{I}}^{t+1}\right\|\right] )  \langle W_{\alpha} {V'_{I}}^{t+1},  W_{\beta} {V'_{I}}^{t+1} \rangle \\
V_{I}^{t+1} & ={V'_{I}}^{t+1} +  L''_{5}( \left[H'^{t+1}_{I}, \left\| W_{\beta} {V'_{I}}^{t+1} \right\|\right] ) {V'_{I}}^{t+1} 
\end{aligned}
\end{equation}
Here $\|\cdot\|: \R^{F\times 3} \rightarrow \R^{F}$ computes the norm for each feature channel. 
$\langle \cdot, \cdot \rangle: \R^{F\times 3} \times \R^{F\times 3} \rightarrow \R^{F}$ is the feature-wise dot product, and $[\cdot, \cdot]: \R^{F\times d_1} \times \R^{F\times d_2} \rightarrow \R^{F\times (d_1+d_2)}$ indicates the concatenation along the second feature channel.
It is easy to show that the update function produces an invariant $H_I^{t+1}$ and an equivariant $V_I^{t+1}$ because 1) dot product and Euclidean norms are invariant under rotation and reflection, and 2) linear combinations of vectors are equivariant. More discussions can be found in \citet{schutt2021equivariant}.

\subsubsection{Implementation of Channel Selection Operations}
\label{appendix:channel_select}

The final prediction of $\dtx$ are compiled from vector channels in $V_{I, \ind(i, C_I)}$. This can be done efficiently by generating pairs of $(I, \ind(i, C_I) )$ used as indices array for advanced indexing in Pytorch. A PyTorch implementation is provided here: 

\begin{lstlisting}[language=Python]
def channel_select(V, m):
    # m: n x 1  CG map for each FG node 
    # V : N x F x 3  vector channels on CG node
    channel_idx = torch.zeros_like( m)
    for cg_type in torch.unique(m): 
        cg_filter = m == cg_type
        # find size of C_I
        k = cg_filter.sum().item() # find size of C_I
        # construct (I, Index(i, C_I))
        channel_idx[cg_filter] = torch.LongTensor(list(range(k))) 
    dx = V[m, channel_idx]
    return dx
\end{lstlisting}

\subsubsection{Ensuring Geometry Self-consistency}
\label{appendix:decoder_compat}
Just as the lift map $\lift$ satisfies $\proj \lift X= X$ , the learnt decoder $\dec$ also needs be compatible with pre-defined surjective CG map $\proj$, i.e. $\proj(\dec(X, z)) = X$. 

\begin{restatable}[]{proposition}{deccompatibility}
$\dec (X, z) = \lift X + \Delta \widetilde{x} - \lift \proj \Delta \widetilde{x}$ satisfies $\proj(\dec(X, z)) = X$. 
 \label{property:deccompatibility}
\end{restatable}
\begin{proof}
From linearity of $\proj$ and the property that $\proj \lift = I$, we have 
    \begin{equation*}
    \begin{aligned}
        \proj(\dec(X,z)) = \proj \lift X + \proj \Delta \widetilde{x} - \proj \lift\proj(\Delta \widetilde{x}) % \\
                     = X +  \proj \Delta \widetilde{x} -\proj\Delta \widetilde{x}%\\
                      = X  
    \end{aligned}
    \end{equation*}  
% This satisfies our definition for $\proj$-compatible decoders. $\blacksquare$    
\end{proof}

\section{Experiments Details}
\label{appendix:experiment}
\subsection{Baselines}
\textbf{Equivariant linear model.}  \citep{wang2019coarse}  We construct the first model following \cite{wang2019coarse}, the reconstructed coordinates are obtained from linear combinations of the coarse coordinates with learnable coefficients $D_{i, I}$.
$
% \begin{equation*}
    \widetilde{x}_{i, m} :=  \sum_I D_{i, I} X_{I, m}
% \end{equation*}
$
where $m \in [3]$ indicates the Cartesian index of the coordinates. To incorporate translational equivariance which the model fails to incorporate, we recenter the FG structure and the mapped CG structure to its geometric center for training and testing of the model. 
\begin{table}[htp]
\centering
\caption{Hyper-parameters for the MLP and equivariant linear baseline models. We used the same set of parameters for both alanine dipeptide and chignolin datasets.
}

\renewcommand{\arraystretch}{1.1}
\begin{tabular}{@{}llll@{}}\toprule
&Alanine Dipeptide & Chignolin \\ \midrule
$\gL_{\textnormal{graph}}$ weight $\gamma$ &5 &5 \\
Training epochs &500 &500 \\
Learning rate &0.0001 &0.0001 \\
Batch size &32 &32 \\
Order for multi-hop graph &2 &2 \\
\bottomrule
\end{tabular}
\label{tab:baseline_param}
\end{table}

\textbf{Multi-Layer Perceptrons.}\citep{an2020machine} We use an MLP regressor as proposed by \cite{an2020machine} to map between CG space to FG space. After flattening the coordinate arrays, the MLP maps $\R^{3N}$ to $\R^{3n}$. The model does not incorporate $E(3)$ equivariance. We use MLPs with 2 hidden layers with hidden dimensions of size $3n$ with $\texttt{ReLU}$ as the activation function. 
% \textbf{Linear Model 2} To readily incorporate $E(3)$ equivariance, we devise the equivariant version of the linear regression with inputs obtained from $\mathcal{B}= \{ X_i - X_j  | (i, j) \in \edge \}$. The prediction for $\Delta x$ are compiled with learn combinations of these vectors with trainable coefficients $a_{i, k}$ and $b_k \in \mathcal{B}$:

% \begin{equation}
%     \Delta x_i = \sum_{b_k \in \mathcal{B}} a_{i, k} b_k
%     \label{eq:equi_linear}
% \end{equation}

% Follow the analysis for CGVAE decoders, the linear model proposed in \cref{eq:equi_linear} is $E(3)$ equivariant. We used the same loss for our baseline methods, the hyperparameters we used can be found in \cref{tab:baseline_param}

% Our experiments show that both linear methods have similar performances, while Linear model 2 requires no additional re-centering. We use the same loss described \ref{section:training-sampling}, the model is optimized with batched gradient descent. 

\begin{table}[htp]
% \tiny
\centering
\caption{hyperparameters used for CGVAE}
\label{tab:hyperparam}
\renewcommand{\arraystretch}{1.1}
\begin{tabular}{@{}l|ll@{}}
\toprule                                                                     & Alanine Dipeptide & Chignolin \\ \midrule
\begin{tabular}[c]{@{}l@{}} Dataset size \end{tabular}                  & 20000             & 10000         \\
\begin{tabular}[c]{@{}l@{}} Edge feature dimension $K$ \end{tabular}         & 8                & 10        \\
\begin{tabular}[c]{@{}l@{}}$\gL_{\textnormal{graph}}$ weight. $\gamma$ \end{tabular}                        & 25.0              & 50.0         \\
\begin{tabular}[c]{@{}l@{}}Encoder conv depth $T^{\textnormal{enc}}$\end{tabular}  & 4                 & 2         \\
\begin{tabular}[c]{@{}l@{}}Prior conv depth $T^{\textnormal{prior}}$\end{tabular}  & 4                 & 2         \\
\begin{tabular}[c]{@{}l@{}}Decoder conv depth $T^{\textnormal{dec}}$ \end{tabular} & 5                 & 9         \\
\begin{tabular}[c]{@{}l@{}} $\dcut$\end{tabular}                      & 8.5               & 12.0       \\
\begin{tabular}[c]{@{}l@{}} $\Dcut$\end{tabular}                      & 9.5              & 25.0      \\
Node embedding dimension $F$                                                                     & 600               & 600      \\
\begin{tabular}[c]{@{}l@{}}Batch  size\end{tabular}                     & 32                & 2         \\
\begin{tabular}[c]{@{}l@{}}Learning  rate\end{tabular}                  & 8e-5              & 1e-4      \\
\begin{tabular}[c]{@{}l@{}}Activation  function\end{tabular}            & swish             & swish     \\
\begin{tabular}[c]{@{}l@{}}Training epochs\end{tabular}                 & 500               & 100        \\
Order for multi-hop graph &2 &2 \\
$\beta$ regularization strength of KL divergence                                             & 0.05              & 0.05   \\ \bottomrule
\end{tabular}
\end{table}

\subsection{Details for CGVAE}
\label{appendix:map_protocl}

\textbf{Parameterize mappings with Auto-Grain.}  Our model works with arbitrary mapping protocols that satisfy the requirement described in \cref{section:cg_def}. In practice, we find mappings that assign close atoms together perform better. We parameterize our CG mapping using Auto-Grain proposed in \cite{wang2019coarse} which readers can refer for details of the method. Auto-Grain parameterizes the discrete assignment matrix $C_{I,i} \in \R^{N \times n}$ in a differentiable way using Gumbel-Softmax transformation \cite{jang2016categorical} with a tunable fictitious temperature $\tau$ for different level of discreteness. The resulting CG transformation using the center of geometries is $M_{I,i} = \frac{C_{I,i}}{\sum_{i=1}^n C_{i,I}}$. A decoding matrix is used to equivariantly reconstruct $x$ based on $Mx$ with minimization of the reconstruction MSD loss. To encourage assignment of closed-together atoms into the same CG bead, we introduce a new geometric loss : $\gL_{\textnormal{geo}} = \frac{1}{n} \sum_{i=1}^n \sum_{k=1}^3 ( (C^T M x)_{i, k} - x_{i,k} )^2$ which penalizes mappings that groups atoms that are far away from each other. $C^T$ can be thought as the lifting operator of the soft-parameterized $M$, satisfying $ M C^{T} = I_{N} $. The total loss for Auto-Grain is $\gL_{MSD} + a \gL_{\textnormal{geo}}$ where $a$ re-weights the two loss contributions. We choose $a = 0.25$ in our experiments. For each experiment, we randomly select 1000 samples to train Auto-Grain to obtain a mapping. We train our model for 1500 epochs, with $\tau$ starting at 1.0 and gradually decreasing its value by 0.001 till it reaches a value of 0.025. The optimization is stochastic, and the obtained mapping will be slightly different, but compatible with CGVAE, for different parameter initializations. We find mapping obtained this way generate higher-quality samples than backbone-based CG protocol (introduced below), especially at low resolution for Chignolin. We present results obtained from CGVAE using the two mapping protocols in \cref{fig:mapping_effect}. 

\textbf{Hyperparameters for CGVAE.}
We use the hyperparameter set described in \cref{tab:hyperparam} for the CGVAE experiments we presented in the main text. The same set of hyperparameters are used for all choices of $N$. For optimization, we use the Adam optimizer \citep{kingma2014adam}. We additionally apply a learning rate scheduler to adaptively reduce learning rate when convergence is reached. We use \texttt{torch.optim.lr\_scheduler.ReduceLROnPlateau} implemented in Pytorch \cite{pytorch} with \texttt{patience=15} and \texttt{factor=0.3}.

\begin{figure*}
\centering
\includegraphics[width=0.77\textwidth]{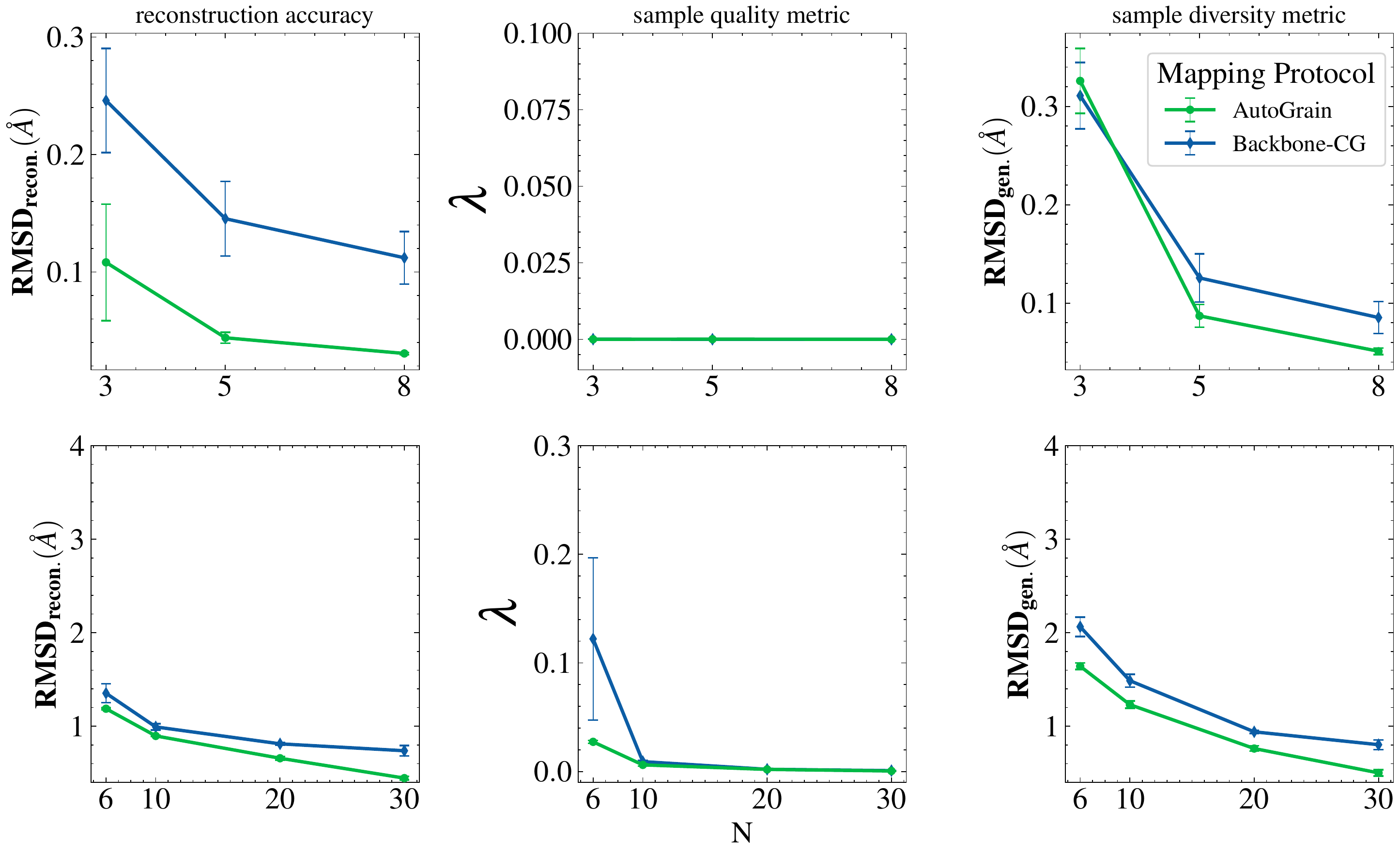}
\caption{Comparing CGVAE results for alanine dipeptide (top) and chignolin (bottom) using two different mapping protocols. In general, the results from the two mappings are comparable, yet Auto-Grain yields better sample qualities for low-resolution Chignolin representations.}
\label{fig:mapping_effect}
\end{figure*}

\textbf{CG based on backbones.} Performing CG transformation based on backbones is an alternative mapping strategy that groups neighboring atoms together. We randomly partition the linear backbone chain of the protein into $N$ segments. We then assign the side chain atoms to the closest backbone segments based on the center of geometry of the backbone segments. We infer the backbone structure of the protein using the \texttt{MDTraj} package \citep{McGibbon2015MDTraj}. After mapping is obtained, we compute $X$ as the center of geometry.

% \textbf{Alanine dipeptide} Given $N$, the cluster map $\map$ is randomly generated by randomly assigning each atom a value in $[1, N]$. We discard $\map$ that do not cover $[1, N]$ to ensure that each CG bead is assigned to at least one atom. We use the center of geometry to compute $X$.

% \textbf{Chignolin} Because chignolin is a relatively large structure,  we find that CG mapping that group atoms based on their closeness works better. To generate mapping that neighboring atoms, we randomly partition the linear backbone chain of the protein into $N$ segments. We then assign the rest of the atoms to the closest backbone segments based on the center of geometry of the backbone segments. We infer the backbone structure of the protein using the \texttt{MDTraj} package \citep{McGibbon2015MDTraj}. We use the center of geometry to compute $X$.

\subsection{Evaluation Metrics }
\label{appendix:eval_metrics}

\textbf{Reconstruction accuracy.} For the reconstruction task, we are interested in the difference between the original reference atomistic structure $x$ and the reconstructed structure $\widetilde{x}$. This informs how well a model can reconstruct samples. For the training of CGVAEs, we remove the stochastic reparameterization to train a deterministic auto-encoders. This choice is to align our evaluation with our deterministic baselines. The discrepancy between $x$ and $\widetilde{x}$ can be evaluated with root mean square distance $\textnormal{RMSD}_{\textnormal{recon}} = \sqrt{\frac{\sum_{i=1}^{n} (x_i - \widetilde{x}_i)^2}{n}}.$

\textbf{Sample qualities.} The sampling task is performed based on the information of CG coordinates only. The invariant features $z \sim p_{\psi}(z| X) $ is then used to generate coordinates with the decoding function $\dec(X, z)$. We are primarily interested in the quality of the generated geometries by evaluating how well  the generated geometry ensembles preserve the bond graph $\mathcal{G}_{\textnormal{gen}}$ compared with the graph $\mathcal{G}_{\textnormal{mol.}}$ induced by the reference structure in the data. We infer $\mathcal{G}_{\textnormal{gen}}$ and $\mathcal{G}_{\textnormal{mol.}}$ with the protocol described in \cref{app:graph_struct}. We quantify the similarity of the two generated graphs by computing the minimum graph edit distance. Because the generated geometries preserve the number of nodes and their identities, we simply need to compute the number of edge addition and removal operations required between $\mathcal{G}_{\textnormal{gen}}$ and $\mathcal{G}_{\textnormal{mol.}}$ and we denote $\mathcal{E}_{\textnormal{gen}}$ and $\mathcal{E}_{\textnormal{mol.}}$ as their edge sets respectively. Because the atom orders of the reference graph and the generated graphs are the same, the graph edit distance is the size of the symmetric difference ($\Delta$), \textnormal{i.e}, the union without the intersection of the two edge sets. We define the quantify $\lambda(\mathcal{E}_{\textnormal{gen}}, \mathcal{E}_{\textnormal{mol.}})$, the graph edit distance normalized by the size of the edge set of the original graph $\mathcal{E}_{\textnormal{mol.}}$.
\begin{equation*}
  \lambda(\mathcal{E}_{\textnormal{gen}}, \mathcal{E}_{\textnormal{mol}}) = \frac{\textnormal{GED}(\mathcal{E}_{\textnormal{gen}}, \mathcal{E}_{\textnormal{mol}})}{|\mathcal{E}_{\textnormal{mol}}|} = \frac{\mathcal{E}_{\textnormal{gen}} \Delta \mathcal{E}_{\textnormal{mol}}}{|\mathcal{E}_{\textnormal{mol}}|}
\end{equation*}
where $\Delta$ denotes the set difference. If $\lambda=0$, it means that the molecular graph of the generated geometries is identical to the reference molecular graph. We evaluate $\lambda$ for graphs for both the all-atom systems and heavy-atom only systems. We make the distinction because heavy-atoms encode most of the important dynamics of the system. Also, substructures that involve hydrogen are highly fluctuating, make the model (especially for models with few $N$) hard to capture its distribution fully. Additionally, we also compute the ratio of valid graphs of the generated samples. Such ratio tends to be low for Chignolin because of its more complex structures. For sample diversity metrics, we only calculate on structures with valid graphs. These values can be found in \cref{tab:dipep} and \cref{tab:chig}. 

\textbf{Sample diversities.} We are also interested in the diversity of sampled geometries for our proposed generative model. Because $\proj$ is a surjective map, each CG configuration corresponds to multiple possible configurations. The more diverse the sampled geometries given $X$, the more information that $z$ has learned for the missing information encoded in $z$. For this purpose, we want to evaluate RMSD compared with the reference geometry to quantify how well the model can generate samples that are not in the reference set. The evaluation is only performed for generated samples with valid molecular graphs. The diversity $\textnormal{RMSD}$ is defined as: 
  $\rmsdgen= \sqrt{\frac{\sum_i^{n} \int p_{\psi}(z|X) (x_i - \mathbf{Dec}(X, z) )^2}{n}}.$
Based on the discussion above, a higher value of $\rmsdgen$ indicates better diversities of generated samples. To compute $\rmsdgen$, we sample 32 values of $z$ for each $X$ in the test set. 

\subsection{Tables and Sample Geometries}

% \subsubsection{Tables}
\label{app:tables}
We show tabulated benchmark results in \cref{tab:dipep} for alanine dipeptide and chignolin datasets respectively for both heavy-atom and all-atom geometries. For  $\rmsdgen$, the higher the value the better the model is in generating diverse samples. To complement the results of heavy atoms shown in \cref{fig:quant_results} for heavy atoms, we also include plots for all-atom reconstruction and sampling in \cref{fig:all_atom_results}. We also compare generated chignolin samples ($N=6$) from different methods in \cref{fig:compare_samples}. 

\begin{figure}[!tbp]
  \centering
  \vspace{-2pt}
    \includegraphics[width=0.8\textwidth]{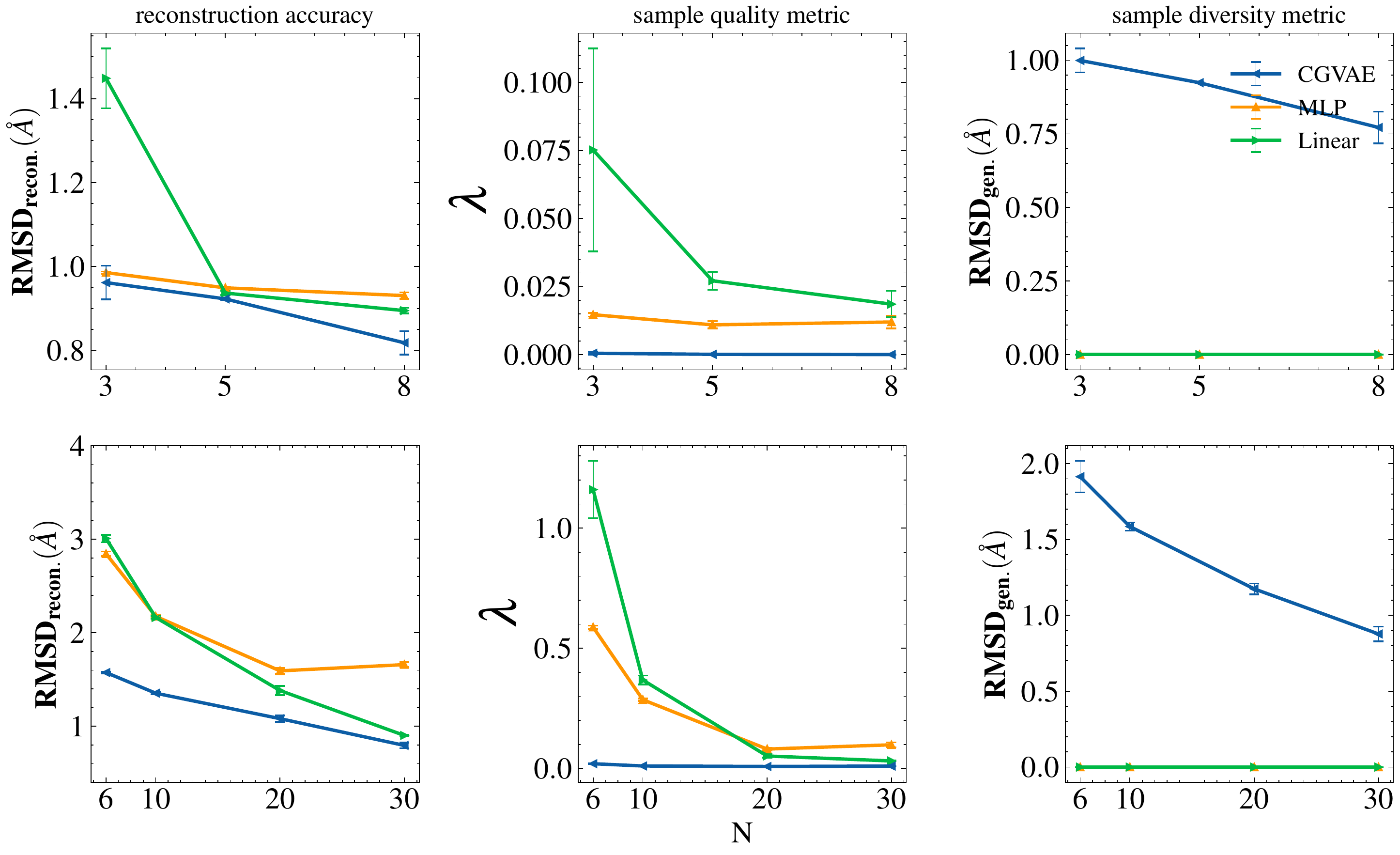}
  \vspace{-2pt}
  \caption{ Benchmarks for all-atom reconstruction and sampled geometries at different resolutions for alanine dipeptide (top) and chignolin (bottom). $\rmsdgen$ is not available for our deterministic baselines because the generation process is deterministic. }
  \vspace{-10pt}
  \label{fig:all_atom_results}
\end{figure}

\begin{figure}
    \centering
    \includegraphics[width=0.66\textwidth]{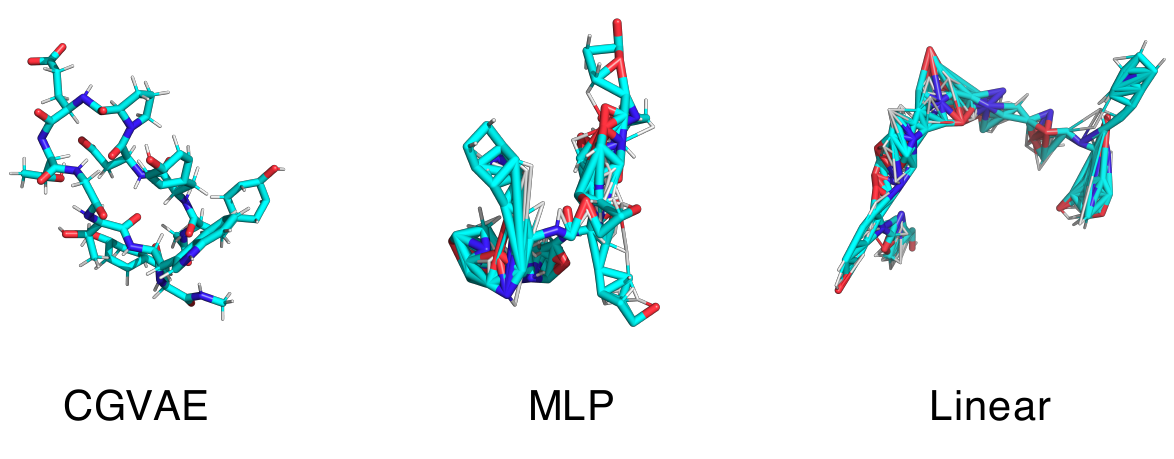}
    \caption{Comparison between generated chignolin samples ($N=6$) from CGVAE and backmapped samples from baseline methods.}
    \label{fig:compare_samples}
\end{figure}

\begin{table}[!htp]\centering
\caption{Alanine dipeptide benchmarks}\label{tab:dipep}
\scriptsize
\begin{tabular}{@{}lllllll@{}}\toprule
\multicolumn{2}{c}{\multirow{2}{*}{Metric}} &\multirow{2}{*}{Model} &\multicolumn{3}{c}{$N$} \\\cmidrule{4-6}
& & &3 &5 &8 \\\midrule
\multirow{9}{*}{Heavy atom} &$\textbf{RMSD}_{\textnormal{recon}}$ &linear &$1.211 \pm 0.075$ &$0.277 \pm 0.032$ &$0.086 \pm 0.009$ \\
& &MLP &$0.439 \pm 0.003$ &$0.288 \pm 0.004$ &$0.269 \pm 0.009$ \\
& &CGVAE &$0.108 \pm 0.050$ &$0.044 \pm 0.005$ &$0.031 \pm 0.001$ \\
\cmidrule{2-6}
&$\lambda$ &linear &$0.081 \pm 0.045$ &$0.001 \pm 0.001$ &$0.000 \pm 0.000$ \\
& &MLP &$0.011 \pm 0.001$ &$0.001 \pm 0.000$ &$0.001 \pm 0.000$ \\
& &CGVAE &$0.000 \pm 0.000$ &$0.000 \pm 0.000$ &$0.000 \pm 0.000$ \\
\cmidrule{2-6}
&$\textbf{RMSD}_{\textnormal{gen}}$ &linear &$0.000 \pm 0.000$ &$0.000 \pm 0.000$ &$0.000 \pm 0.000$ \\
& &MLP &$0.000 \pm 0.000$ &$0.000 \pm 0.000$ &$0.000 \pm 0.000$ \\
& &CGVAE &$0.326 \pm 0.033$ &$0.087 \pm 0.012$ &$0.051 \pm 0.003$ \\
\cmidrule{2-6}
&Valid graph ratio &linear &$0.561 \pm 0.153$ &$0.990 \pm 0.009$ &$1.000 \pm 0.000$ \\
& &MLP &$0.908 \pm 0.011$ &$0.992 \pm 0.001$ &$0.993 \pm 0.002$ \\
& &CGVAE &$1.000 \pm 0.000$ &$1.000 \pm 0.000$ &$1.000 \pm 0.000$ \\
\midrule
\multirow{9}{*}{All atoms} &$\textbf{RMSD}_{\textnormal{recon}}$ &linear &$1.448 \pm 0.071$ &$0.937 \pm 0.005$ &$0.895 \pm 0.006$ \\
& &MLP &$0.986 \pm 0.002$ &$0.949 \pm 0.003$ &$0.931 \pm 0.008$ \\
& &CGVAE &$0.962 \pm 0.040$ &$0.923 \pm 0.002$ &$0.818 \pm 0.028$ \\
\cmidrule{2-6}
&$\lambda$ &linear &$0.075 \pm 0.037$ &$0.027 \pm 0.003$ &$0.019 \pm 0.005$ \\
& &MLP &$0.015 \pm 0.001$ &$0.011 \pm 0.001$ &$0.012 \pm 0.002$ \\
& &CGVAE &$0.000 \pm 0.000$ &$0.000 \pm 0.000$ &$0.000 \pm 0.000$ \\
\cmidrule{2-6}
&$\textbf{RMSD}_{\textnormal{gen}}$ &linear &$0.000 \pm 0.000$ &$0.000 \pm 0.000$ &$0.000 \pm 0.000$ \\
& &MLP &$0.000 \pm 0.000$ &$0.000 \pm 0.000$ &$0.000 \pm 0.000$ \\
& &CGVAE &$1.000 \pm 0.041$ &$0.924 \pm 0.001$ &$0.772 \pm 0.054$ \\
\cmidrule{2-6}
&Valid graph ratio &linear &$0.230 \pm 0.151$ &$0.667 \pm 0.019$ &$0.802 \pm 0.048$ \\
& &MLP &$0.732 \pm 0.015$ &$0.808 \pm 0.023$ &$0.790 \pm 0.038$ \\
& &CGVAE &$0.990 \pm 0.010$ &$0.998 \pm 0.001$ &$0.999 \pm 0.000$ \\
\bottomrule
\end{tabular}
\end{table}

\begin{table}[!htp]\centering
\caption{Chignolin benchmarks}\label{tab:chig}
\scriptsize
\begin{tabular}{@{}llllllll@{}}\toprule
\multicolumn{2}{c}{\multirow{2}{*}{Metric}} &\multirow{2}{*}{Model} &\multicolumn{4}{c}{N} \\\cmidrule{4-7}
& & &6 &10 &20 &30 \\\midrule
\multirow{12}{*}{heavy atom} &\multirow{3}{*}{$\textbf{RMSD}_{\textnormal{recon}}$} &linear &$2.723 \pm 0.055$ &$1.845 \pm 0.013$ &$1.077 \pm 0.036$ &$0.702 \pm 0.011$ \\
& &MLP &$2.558 \pm 0.026$ &$1.894 \pm 0.013$ &$1.318 \pm 0.030$ &$1.388 \pm 0.024$ \\
& &CGVAE &$1.188 \pm 0.010$ &$0.897 \pm 0.007$ &$0.658 \pm 0.020$ &$0.445 \pm 0.019$ \\
\cmidrule{2-7}
&\multirow{3}{*}{$\lambda$} &linear &$1.124 \pm 0.127$ &$0.411 \pm 0.028$ &$0.085 \pm 0.011$ &$0.021 \pm 0.002$ \\
& &MLP &$0.611 \pm 0.009$ &$0.331 \pm 0.008$ &$0.125 \pm 0.008$ &$0.146 \pm 0.010$ \\
& &CGVAE &$0.027 \pm 0.001$ &$0.006 \pm 0.001$ &$0.002 \pm 0.000$ &$0.001 \pm 0.000$ \\
\cmidrule{2-7}
&\multirow{3}{*}{$\textbf{RMSD}_{\textnormal{gen}}$} &linear &$0.000 \pm 0.000$ &$0.000 \pm 0.000$ &$0.000 \pm 0.000$ &$0.000 \pm 0.000$ \\
& &MLP &$0.000 \pm 0.000$ &$0.000 \pm 0.000$ &$0.000 \pm 0.000$ &$0.000 \pm 0.000$ \\
& &CGVAE &$1.642 \pm 0.034$ &$1.231 \pm 0.038$ &$0.762 \pm 0.028$ &$0.502 \pm 0.033$ \\
\cmidrule{2-7}
&\multirow{3}{*}{valid graph ratio} &linear &$0.000 \pm 0.000$ &$0.000 \pm 0.000$ &$0.002 \pm 0.002$ &$0.082 \pm 0.022$ \\
& &MLP &$0.000 \pm 0.000$ &$0.000 \pm 0.000$ &$0.000 \pm 0.000$ &$0.000 \pm 0.000$ \\
& &CGVAE &$0.415 \pm 0.032$ &$0.716 \pm 0.034$ &$0.843 \pm 0.015$ &$0.955 \pm 0.012$ \\
\cmidrule{2-7}
\multirow{12}{*}{all atom} &\multirow{3}{*}{$\textbf{RMSD}_{\textnormal{recon}}$} &linear &$3.008 \pm 0.040$ &$2.163 \pm 0.013$ &$1.382 \pm 0.050$ &$0.902 \pm 0.007$ \\
& &MLP &$2.844 \pm 0.026$ &$2.175 \pm 0.020$ &$1.592 \pm 0.031$ &$1.659 \pm 0.027$ \\
& &CGVAE &$1.574 \pm 0.009$ &$1.353 \pm 0.012$ &$1.081 \pm 0.033$ &$0.795 \pm 0.029$ \\
\cmidrule{2-7}
&\multirow{3}{*}{$\lambda$} &linear &$1.160 \pm 0.119$ &$0.368 \pm 0.018$ &$0.052 \pm 0.007$ &$0.031 \pm 0.002$ \\
& &MLP &$0.588 \pm 0.006$ &$0.286 \pm 0.005$ &$0.081 \pm 0.005$ &$0.098 \pm 0.011$ \\
& &CGVAE &$0.020 \pm 0.001$ &$0.010 \pm 0.001$ &$0.008 \pm 0.002$ &$0.010 \pm 0.000$ \\
\cmidrule{2-7}
&\multirow{3}{*}{$\textbf{RMSD}_{\textnormal{gen}}$} &linear &$0.000 \pm 0.000$ &$0.000 \pm 0.000$ &$0.000 \pm 0.000$ &$0.000 \pm 0.000$ \\
& &MLP &$0.000 \pm 0.000$ &$0.000 \pm 0.000$ &$0.000 \pm 0.000$ &$0.000 \pm 0.000$ \\
& &CGVAE &$1.914 \pm 0.103$ &$1.586 \pm 0.027$ &$1.175 \pm 0.036$ &$0.877 \pm 0.049$ \\
\cmidrule{2-7}
&\multirow{3}{*}{valid graph ratio} &linear &$0.000 \pm 0.000$ &$0.000 \pm 0.000$ &$0.000 \pm 0.000$ &$0.000 \pm 0.000$ \\
& &MLP &$0.000 \pm 0.000$ &$0.000 \pm 0.000$ &$0.000 \pm 0.000$ &$0.000 \pm 0.000$ \\
& &CGVAE &$0.072 \pm 0.027$ &$0.133 \pm 0.024$ &$0.278 \pm 0.099$ &$0.132 \pm 0.011$ \\
\bottomrule
\end{tabular}
\end{table}

\subsection{On Pseudoscalar Initialization for $N = 3$}
\label{app:pseudo_init}
As discussed in the main text, for $N=3$, the vector basis sets are confined in a plane, so it is required to initialize pseudoscalar update with non-zero values to break the reflection symmetry. As it is shown in \cref{fig:pseudo_init}, the decoder with pseudoscalar initialized with $0_{F}$ produces coordinates that are confined in a plane, while a non-zero pseudoscalar initialization produces proper geometries that span the 3D space. For the case $N > 3$, it is less likely for CG beads to be in the same plane. However, when it is needed, one can always initialize pseudoscalar as non-zeros to break the symmetry even for larger $N$ cases. 
% As discussed in the main text, for coarse representations for $N=3$, the set of equivariant basis does not span the full 3D space. As a results, the decoded geometries lives in a 2D plane if CGVAE is used. Augmented with cross product updates, CGVAE-SE(3) can decode with geometry basis set that spans the 3D space. This is illustrated in \cref{fig:se3vse3} where we show the generated geometries alanine dipeptide with $N=3$ representations. It shows the advantage of using cross updates for the very coarse representation of the conformation. 

\begin{figure}
    \centering
    \includegraphics[width=0.8\textwidth]{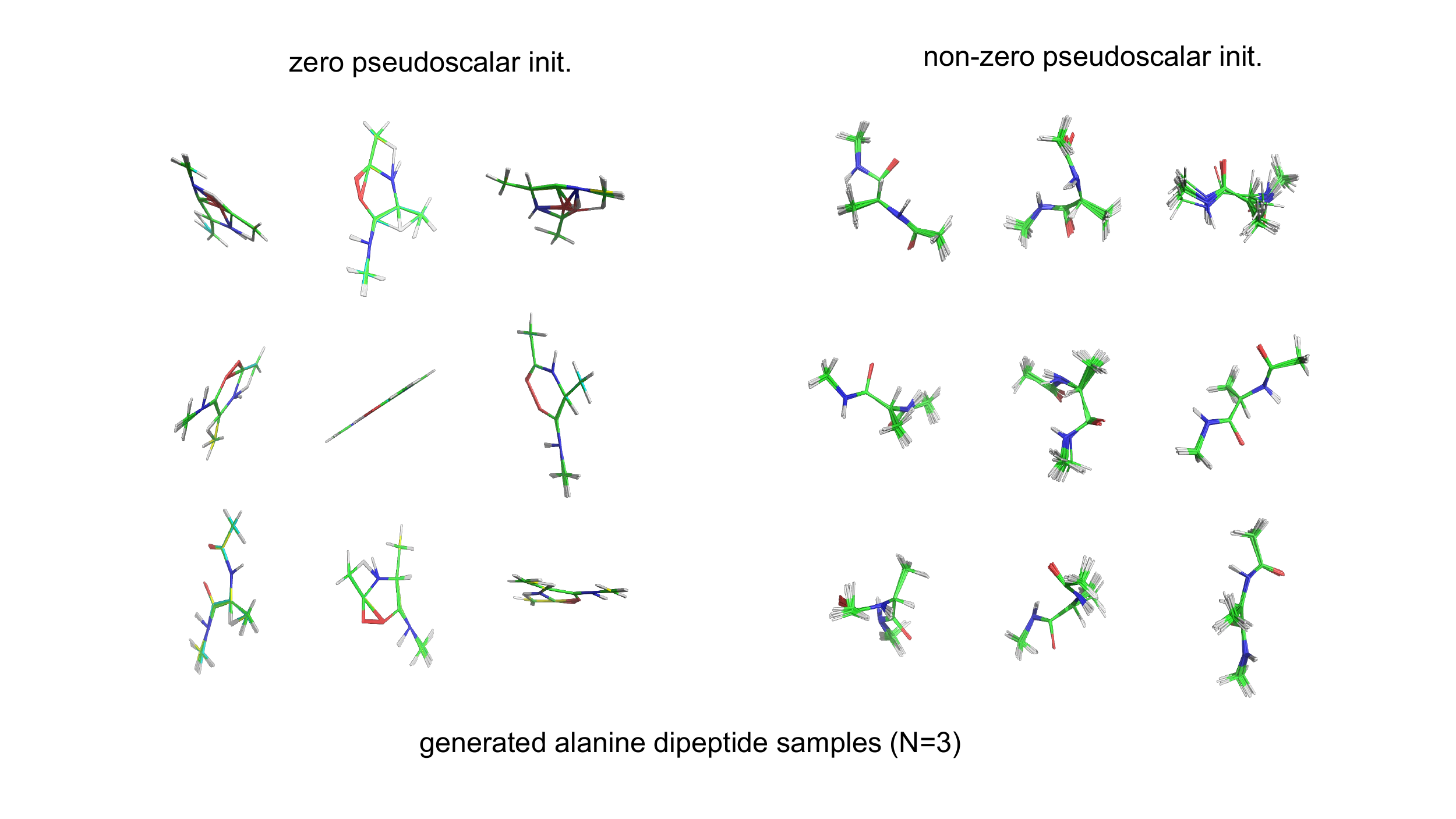}
    \caption{Comparing zero and non-zero pseudoscalar initializations for CGVAE at the resolution of $N=3$ for alanine dipeptide. }
    \label{fig:pseudo_init}
\end{figure}

% \subsubsection{Experiments on reflection equivariance }
% \label{app:reflection_equi_exp}
% To test the reflection equivariance properties of our $SE(3)$ and $E(3)$ decoders, we train both decoders on the original datasets, and randomly rotate and reflect both CG and FG geometries in the test set and feed the reflected inputs into the model. We show our results in in \cref{fig:reflection_equivariance}. Our experiments show that the $E(3)$ model generalizes well to the reflected geometries, while the $SE(3)$ decoder does a worse job. This can also be seen in samples shown in the figure.

% \begin{figure*}
%     \centering
    
%   \subfloat{\includegraphics[width=0.6\textwidth]{figures/reflection_equivariance_sample.pdf}\label{fig:f1}}
%   \hfill
%   \subfloat{\includegraphics[width=0.4\textwidth]{figures/reflection_equivariance_plot.pdf}\label{fig:f2}}
%     % \includegraphics[width=0.75\textwidth]{figures/reflection_equivariance_sample.pdf}
%     \caption{Visualizing reflection equivariance of our decoders. 
%     After training on the alanine dipeptide data, we test the model's reconstruction (\textcolor{green}{green}) with the reflected coordinate inputs (\textcolor{orange}{orange}). The $E(3)$ decoder aligns with the reflected geometries better because it is reflection equivariant. The train and test reconstruction metrics show that the $E(3)$ decoder generalize to geometries that are reflected, while the $SE(3)$ decoder does not.}
%     \label{fig:reflection_equivariance}
% \end{figure*}

\end{document}